\documentclass{article}

\usepackage{microtype}
\usepackage{graphicx}
\usepackage{subfigure}
\usepackage{booktabs} 

\usepackage{hyperref}

\usepackage[accepted]{icml2020}
\def\preprint

\usepackage[utf8]{inputenc} 
\usepackage[T1]{fontenc}    
\usepackage{url}            
\usepackage{amsfonts}       
\usepackage{nicefrac}       
\usepackage{microtype}      

\usepackage{amsmath}
\usepackage{amssymb}
\usepackage{amsthm}
\usepackage{color}
\usepackage{mathtools}

\newcommand{\ve}[1]{\ensuremath{\mathbf{#1}}}
\newcommand{\vet}[2]{\ensuremath{\mathbf{#1}_{#2}}}
\newcommand{\vetl}[3]{\ensuremath{\mathbf{#1}_{#2}^{(#3)}}}

\newcommand{\veb}[2]{\ensuremath{\ve{#1}^{\{#2\}}}}

\newcommand{\ves}[2]{\ensuremath{\mathbf{#1}_{\mathbf{#2}}}}
\newcommand{\vesl}[3]{\ensuremath{\ves{#1}{#2}^{(#3)}}}
\newcommand{\vehs}[2]{\ensuremath{\mathbf{\hat{#1}}_\mathbf{#2}}}
\newcommand{\vehsb}[3]{\ensuremath{\vehs{#1}{#2}^{\{#3\}}}}
\newcommand{\vehtb}[3]{\ensuremath{\veht{#1}{#2}^{\{#3\}}}}

\newcommand{\veht}[2]{\ensuremath{\mathbf{\hat{#1}}_{#2}}}

\newcommand{\ma}[1]{\ensuremath{\mathbf{#1}}}

\newcommand{\mas}[2]{\ensuremath{\mathbf{#1}_{\mathbf{#2}}}}

\newcommand{\masl}[3]{\ensuremath{\mas{#1}{#2}^{(#3)}}}

\newcommand{\mahs}[2]{\ensuremath{\mathbf{\hat{#1}}_{\mathbf{#2}}}}

\newcommand{\mahsb}[3]{\ensuremath{\mahs{#1}{#2}^{\{#3\}}}}

\newcommand{\f}[1]{\ensuremath{f\left(#1\right)}}
\newcommand{\fopen}[1]{\ensuremath{f\left(#1\right.}}
\newcommand{\fclose}[1]{\ensuremath{\left.#1\right)}}

\newcommand{\fsigmoid}[1]{\ensuremath{\sigma\left(#1\right)}}
\newcommand{\fsigopen}[1]{\ensuremath{\sigma\left(#1\right.}}
\newcommand{\fsigclose}[1]{\ensuremath{\left.#1\right)}}
\newcommand{\ftanh}[1]{\ensuremath{\tanh\left(#1\right)}}
\newcommand{\ftanhopen}[1]{\ensuremath{\tanh\left(#1\right.}}
\newcommand{\ftanhclose}[1]{\ensuremath{\left.#1\right)}}

\newcommand{\g}[1]{\ensuremath{g\left(#1\right)}}
\newcommand{\rf}[1]{\ensuremath{r\left(#1\right)}}

\newcommand{\fb}[2]{\ensuremath{f^{\{#1\}}\left(#2\right)}}

\newcommand{\real}[1]{\ensuremath{\mathbb{R}^{{#1}}}}

\newcommand{\set}[1]{\ensuremath{\left\{#1\right\}}}

\newcommand{\On}[1]{\ensuremath{\mathcal{O} \left({#1}\right)}}

\newtheorem{theorem}{Theorem}
\newtheorem{lemma}{Lemma}

\begin{document}
\twocolumn[
\icmltitle{Approximating Stacked and Bidirectional Recurrent Architectures with the Delayed Recurrent Neural Network}



\icmlsetsymbol{equal}{*}

\begin{icmlauthorlist}
		\icmlauthor{Javier S.~Turek}{intel}	
		\icmlauthor{Shailee Jain}{csuta}		
		\icmlauthor{Vy A.~Vo}{intel}		
		\icmlauthor{Mihai Capot\u{a}}{intel}		
		\icmlauthor{Alexander G.~Huth}{csuta,neurouta}		
		\icmlauthor{Theodore L.~Willke}{intel}		
\end{icmlauthorlist}

\icmlaffiliation{intel}{Intel Labs, Hillsboro, Oregon, USA}	
\icmlaffiliation{csuta}{Department of Computer Science, The University of Texas at Austin, Austin, Texas, USA}		
\icmlaffiliation{neurouta}{Department of Neuroscience, The University of Texas at Austin, Austin, Texas, USA}

\icmlcorrespondingauthor{Javier S.~Turek}{javier.turek@intel.com}

\icmlkeywords{RNN, Bidirectional RNN, LSTM, Bidirectional LSTM, delay}

\vskip 0.3in
]



\printAffiliationsAndNotice{}  

\begin{abstract}

Recent work has shown that topological enhancements to recurrent neural networks (RNNs) can increase their expressiveness and representational capacity. Two popular enhancements are stacked RNNs, which increases the capacity for learning non-linear functions, and bidirectional processing, which exploits acausal information in a sequence. In this work, we explore the delayed-RNN, which is a single-layer RNN that has a delay between the input and output. We prove that a weight-constrained version of the delayed-RNN is equivalent to a stacked-RNN. We also show that the delay gives rise to partial acausality, much like bidirectional networks. Synthetic experiments confirm that the delayed-RNN can mimic bidirectional networks, solving some acausal tasks similarly, and outperforming them in others. Moreover, we show similar performance to bidirectional networks in a real-world natural language processing task. These results suggest that delayed-RNNs can approximate topologies including stacked RNNs, bidirectional RNNs, and stacked bidirectional RNNs -- but with equivalent or faster runtimes for the delayed-RNNs.
\end{abstract}

\section{Introduction}
\label{sec:intro}

Recurrent neural networks (RNN) have successfully been used for sequential tasks like language modeling \cite{sutskever2011generating}, machine translation \cite{seq2seq}, and speech recognition \cite{amodei2016deep}. They approximate complex, non-linear temporal relationships by maintaining and updating an internal state for every input element. However, they face several challenges while modeling long-term dependencies, motivating work on variant architectures.

Firstly, due to the long credit assignment paths in RNNs, the gradients might vanish or explode \cite{bengio_frasconi_1994}. This has led to gated variants like the Long Short-term Memory (LSTM) \cite{LSTM} that can retain information over long timescales. 
Secondly, it is well known that deeper networks can more efficiently approximate a broader range of functions \cite{bengio2007_depth, bianchini_2014}. While RNNs are deep \textit{in time}, they are limited in the number of non-linearities applied to recent inputs. 

To increase depth, there has been extensive work on \textit{stacking} RNNs into multiple layers \cite{stacking_schmidhuber1992, bengio_depth}. In vanilla stacked RNNs, each layer applies a non-linearity and passes information to the next layer, while also maintaining a recurrent connection to itself. 
To effectively propagate gradients across the hierarchy, skip or shortcut connections can be used \cite{shortcut_raiko,stacking_graves}.
Alternatives like recurrent highway networks \cite{rhn} introduce non-linearities between timesteps through ``micro-ticks" \cite{act}. Pascanu et al.~\yrcite{pascanu_2014} increase depth by adding feedforward layers between state-to-state transitions. Gated feedback networks \cite{All2All} allow for layer-to-layer interactions between adjacent timesteps. All these variants thus introduce topological modifications to retain information over longer timescales and model hierarchical temporal dependencies.

Another development is the bidirectional RNN (Bi-RNN) \cite{BiRNN,BiLSTM}. While RNNs are inherently causal, Bi-RNNs model acausal interactions by processing sequences in both forward and backward directions. They achieve state-of-the-art performance on parts-of-speech tagging \cite{POSDualBiLSTM} and sentiment analysis \cite{sentimentanalysis}, demonstrating that some natural language processing (NLP) tasks benefit greatly from combining past and future inputs. 

The successes of these RNN architectural variants seem to derive from two common properties: depth and acausality.
In this paper we investigate the \textbf{delayed-recurrent neural network (d-RNN)}, an extremely simple variant that adds both depth and acausality to the RNN. The d-RNN is a single-layer RNN that imposes depth in time by delaying the output of the model. 
We analyze the d-RNN and prove that when it is constrained with sparse weights, the model is equivalent to a stacked RNN. 
Further, noting that the delay introduces acausal processing, we use a d-RNN to approximate bidirectional recurrent networks.
We show empirically that a d-RNN has the capability to solve some tasks similarly to stacked and bidirectional RNNs, and outperform them in others. Additionally, we show that even if the d-RNN approximation carries some error, this model can provide much faster runtimes than alternatives.

\section{Background}
\label{sec:background}
Given a sequential input $\set{\vet{x}{t}}_{t=1...T}, \vet{x}{t} \in \real{q}$, a single-layer RNN is defined by: 
\begin{align}
\veht{h}{t} &= \f{\mahs{W}{x} \vet{x}{t} + \mahs{W}{h} \veht{h}{t-1} + \vehs{b}{h}}, \label{eq:RNN_h} \\
\veht{y}{t} &= \g{\mahs{W}{o} \veht{h}{t} + \vehs{b}{o}}, \label{eq:RNN_y}
\end{align}
where \f{\cdot} and \g{\cdot} are element-wise activation function such as $\tanh$ and $\mathrm{softmax}$, $\veht{h}{t} \in \real{n}$ is the hidden state at timestep $t$ with $n$ units, and $\veht{y}{t} \in \real{m}$ is the network output.
Learned parameters include input weights \mahs{W}{x}, recurrent weights \mahs{W}{h}, bias term \vehs{b}{h}, output weights \mahs{W}{o}, and bias term \vehs{b}{o}. The initial hidden state is denoted \veht{h}{0}.

Stacked recurrent units are typically used to provide depth in RNNs \cite{stacking_schmidhuber1992, bengio_depth}. 
Based on Eq. \eqref{eq:RNN_h} and \eqref{eq:RNN_y}, a stacked RNN with $k$ layers is given by:
\begin{align}
\vetl{h}{t}{1} &= \f{\masl{W}{x}{1} \vet{x}{t} + \masl{W}{h}{1} \vetl{h}{t-1}{1} + \vesl{b}{h}{1}},\; i=1 \label{eq:sRNN_h1}\\
\vetl{h}{t}{i} &= \f{\masl{W}{x}{i} \vetl{h}{t}{i-1} + \masl{W}{h}{i} \vetl{h}{t-1}{i} + \vesl{b}{h}{i}},\;  i=2\dots k \label{eq:sRNN_hk}\\
\vet{y}{t} &= \g{\mas{W}{o} \vetl{h}{t}{k} + \ves{b}{o}}, \label{eq:sRNN_y}
\end{align}
where the activation function and parameterization follow the single-layer RNN. Separate weights and bias terms for each layer $i$ are given by \masl{W}{x}{i}, \masl{W}{h}{i}, and \vesl{b}{h}{i}. The hidden state for this layer at timestep $t$ is \vetl{h}{t}{i}. The stacked RNN has initial hidden state vectors $\vetl{h}{0}{1}\dots\vetl{h}{0}{k}$ corresponding to the $k$ layers. The hat operator is used for vectors and matrices in the single-layer RNN, while those without are for the stacked RNN.

\section{Delayed-Recurrent Neural Network}
\label{sec:dRNN}
One way to increase depth in RNNs is to stack recurrent layers, as suggested above. An alternative is to consider time as a means to increase depth within a single-layer RNN. However, single-layer RNNs are limited in the number of non-linearities applied to recent inputs: there is a single non-linearity between the most recent input \vet{x}{t} and its respective output \veht{y}{t}. Previous efforts \cite{pascanu_2014,act,rhn} overcame this limitation by incorporating intermediate non-linearities between input elements in different ways. These solutions add computational steps between elements in the sequence, greatly increasing runtime complexity.
In this work, we explore the delayed-recurrent neural network (d-RNN), in which effective depth is increased by introducing a ``delay''  between the input and output. 

Formally, we define a d-RNN to be a single-layer recurrent neural network as in Equations \eqref{eq:RNN_h} and \eqref{eq:RNN_y}, such that for any input \vet{x}{t} the respective output is obtained in \veht{y}{t+d}, i.e., $d$ timesteps later (Figure \ref{fig:dRNN}). We refer to $d$ as the ``delay'' of the network. The initial hidden state, \veht{h}{0}, for a d-RNN is initialized in the same manner as an RNN.

\begin{figure}[tb]
	\vskip 0.2in
	\begin{center}
	\centerline{\includegraphics[width=\columnwidth]{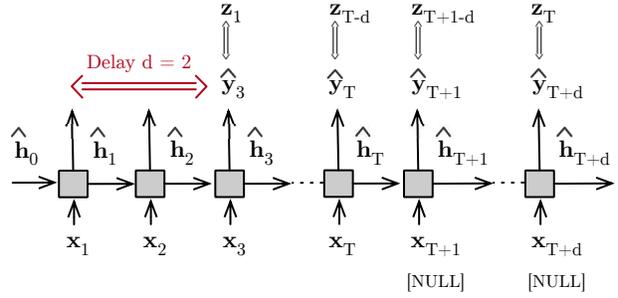}}
	\caption{A delayed-recurrent neural network (d-RNN) processing a sequence of $T$ elements. The output is delayed by $d=2$ timesteps. The first output element is in \veht{y}{3} and the last in \veht{y}{T+d}. The input sequence has $d$ additional elements, such as `[NULL]' symbols. During training, the outputs are compared with the $T$ elements of the labeled sequence $\{\vet{z}{j}\}_j$.}
	\label{fig:dRNN}
\end{center}
	\vskip -0.2in
\end{figure}
Delaying the output requires special considerations on the data that differ slightly from an RNN.
Input sequences need to have $T+d$ elements instead of $T$. Depending on the task being solved, this can be achieved by adding a ``null'' input element (e.g., the zero vector), or including $d$ additional elements in the input sequence.
When doing a forward pass over the d-RNN for inference, outputs from $t=1$ to $d$ are discarded as we expect the output for \vet{x}{1} to be at \veht{y}{1+d}. The output sequence goes from $\veht{y}{1+d} $ to \veht{y}{T+d}, and has $T$  elements.

Training loss is computed by comparing \vet{z}{t}, the expected output for input \vet{x}{t}, with \veht{y}{t+d}. Thus, gradients are back-propagated only from delayed outputs $\veht{y}{1+d}, \dots ,\veht{y}{T+d}$. In this way, any modified recurrent cell, such as an LSTM or GRU, can be trained with delayed output to obtain a delayed version of the architecture, e.g., d-LSTM or d-GRU.

\subsection{Complexity}
\label{sec:complexity}
Consider an RNN with $n$ units, where input elements have dimension $q$, and output elements have dimension $m$.
Computing one timestep of this RNN requires three matrix-vector multiplications with complexity \On{nq + nm+ n^2 }. Applying the non-linear functions \f{\cdot} and \g{\cdot} requires \On{m + n}. Hence, each step of this RNN has runtime complexity of \On{nq+nm + n^2}. For a sequence of length $T$, the overall computational effort is \On{T(nq+nm + n^2)}. For a d-RNN, the number of timesteps is increased by the delay $d$, giving total runtime complexity of \On{(T+d)(nq+nm + n^2)}. 

While the d-RNN incurs some cost, it is cheaper than alternative methods such as micro-steps \cite{act,rhn}, where additional timesteps are inserted between each pair of elements in both the input and output sequences. The runtime complexity for each micro-step is similar to an RNN step, leading the micro-step model complexity to grow with the number of micro-steps $d$ proportionally to \On{d T}. In contrast, the d-RNN model complexity only grows proportionally to \On{d + T}.
 
\subsection{Stacked RNNs are d-RNNs}
The mathematical structure of a stacked RNN is similar to a single-layer RNN with the addition of between-layer connections that add depth. Here we show that any stacked RNN can be flattened into a single-layer d-RNN that produces the exact sequence of hidden states and outputs. We exchange the depth from the between-layer connections with temporal depth applied through a delay in the output. To illustrate this, we rewrite the parameters of a single-layer RNN using the weights and bias terms of a $k$-layer stacked RNN from Equations \eqref{eq:sRNN_h1}-\eqref{eq:sRNN_y}:

\begin{align}
\mahs{W}{h}  &=   \begin{bmatrix} 
\masl{W}{h}{1}  & \ma{0} & & \cdots& & \ma{0} \\ 
\masl{W}{x}{2}  & \masl{W}{h}{2}  & & & & \\
\ma{0}& \ddots & \ddots & \ddots & & \vdots\\
\vdots& \ddots & \masl{W}{x}{i}  & \masl{W}{h}{i} & \ddots & \\
& & & \ddots & \ddots& \ma{0}\\
\ma{0}& \cdots& & \ma{0} & \masl{W}{x}{k}  & \masl{W}{h}{k}  
\end{bmatrix} \label{eq:dRNN_Wh}, \\ 
\vehs{b}{h} & =  \begin{bmatrix} 
\vesl{b}{h}{1} \\ 
\vdots \\
\vesl{b}{h}{k}
\end{bmatrix}, \qquad \qquad 
\mahs{W}{x}  =  \begin{bmatrix} 
\masl{W}{x}{1} \\ 
\ma{0} \\
\vdots \\
\ma{0}
\end{bmatrix}, \label{eq:dRNN-all} \\ 
\mahs{W}{o}  &= \begin{bmatrix} 
\ma{0} & 
\cdots  &
\ma{0} & \mas{W}{o} 
\end{bmatrix}, \; \quad 
\vehs{b}{o}  =  \ves{b}{o}, &\label{eq:dRNN_o} 
\end{align}
where $\mahs{W}{x}\in\real{kn \times q}$ are the input weights, $\mahs{W}{h}\in\real{kn \times kn}$ the recurrent weights,  $\vehs{b}{h}\in\real{kn}$ the biases, $\mahs{W}{o}\in\real{m \times kn}$ the output weights, and $\vehs{b}{o}\in\real{m}$ the output biases.

One can see from Eq.~\eqref{eq:dRNN_Wh}-\eqref{eq:dRNN_o} that each layer in the stacked RNN is converted into a group of units in the single-layer RNN. The block bidiagonal structure of the recurrent weight matrix \mahs{W}{h} makes the hidden state act as a buffer, where each group of units only receives input from itself and the previous group. Information processed through this buffering mechanism eventually arrives at the output after $k-1$ timesteps. In fact, the obtained model is a d-RNN with delay $d=k-1$ and sparsely constrained weights. Note that the d-RNN performs the same computations as the stacked version by trading depth \textit{in layers} for depth \textit{in time}.

Next, we define the following notation: for a vector $\ve{v} \in \real{kn}$ with $k$ blocks, the subvector $\veb{v}{i}\in\real{n}$ refers to its $i$th block following the partition from Equations \eqref{eq:dRNN_Wh}-\eqref{eq:dRNN_o}. We now prove that a d-RNN parametrized by Eq.~\eqref{eq:dRNN_Wh}-\eqref{eq:dRNN_o}  is exactly equivalent to the stacked RNN in Eqs. \eqref{eq:sRNN_h1}-\eqref{eq:sRNN_y}. The proof can be extended to more complex recurrent cells. We include a proof for LSTMs in the supplementary material.
\begin{theorem}
	\label{thm:dRNN}
	Given an input sequence $\set{\vet{x}{t}}_{t=1...T}$ and a stacked RNN with $k$ layers defined by Equations \eqref{eq:sRNN_h1}-\eqref{eq:sRNN_y} with activation functions \f{\cdot} and \g{\cdot}, and initial states $\{\vetl{h}{0}{i}\}_{i=1\dots k}$, the d-RNN with delay $d=k-1$, defined by Equations \eqref{eq:dRNN_Wh}-\eqref{eq:dRNN_o} and initialized with \veht{h}{0} such that $\vehtb{h}{i-1}{i}=\vetl{h}{0}{i}, \; \forall i=1\dots k$, produces the same output sequence but delayed by $k-1$ timesteps, i.e., $\veht{y}{t+k-1} = \vet{y}{t}$ for all $t=1\dots T$. Further, the sequence of hidden states at each layer $i$ are equivalent with delay $i-1$, i.e., $\vehtb{h}{t+i-1}{i}  = \vetl{h}{t}{i}$ for all $1 \le i \le k $ and $t\ge1$.
\end{theorem}
\begin{proof}
	See Section 1 of the supplementary material.
\end{proof}

Theorem \ref{thm:dRNN} makes an assumption that  \veht{h}{0}  in the d-RNN can be initialized such that it achieves $\vehtb{h}{i-1}{i}=\vetl{h}{0}{i}$ for all blocks. 
Lemma \ref{thm:initializationRNN} below implies that initialization for the d-RNN with constrained weights can always be computed from the stacked RNN.
The intuition behind it is that we can compute recursively from $\vehtb{h}{i-1}{i}=\vetl{h}{0}{i}$ to \vehtb{h}{0}{i} for block $i$, while inverting the activation function. All commonly used activation functions are surjective, thus it is enough to know the right-inverse of the activation function \f{\cdot} (see proof of Lemma). For example, when \f{\cdot} is the ReLU, the right-inverse is the identity function $\rf{d} = d$.
\begin{lemma}
	\label{thm:initializationRNN}
	Let $f:\real{}\rightarrow D$ be a surjective activation function that maps elements in \real{} to elements in interval $D$. Also, let $\vetl{h}{0}{i} \in D^{n}$ for $i=1\dots k$ be the hidden state initialization for a stacked RNN with $k$ layers as defined in \eqref{eq:sRNN_h1}-\eqref{eq:sRNN_hk}.
	Then, there exists an initial hidden state vector $\veht{h}{0} \in \real{kn}$ for a single-layer network in Equations \eqref{eq:dRNN_Wh}-\eqref{eq:dRNN-all} such that $\vehtb{h}{i-1}{i}=\vetl{h}{0}{i} \; \forall i=1\dots k$.
\end{lemma}
\begin{proof}
	See Section 2 of the supplementary material.
\end{proof}

From this theorem we see that $k$-layer stacked RNNs can be perfectly expressed as a single-layer d-RNN. In this case, the d-RNN has a specific sparsity structure in its weight matrices that is not present in the generic RNN or d-RNN.
As the stacked RNN and the d-RNN with sparsely constrained weights models are equivalent, there is no difference in favor of which one is used in practice, and their runtime complexities are the same\footnote{Their runtime complexities are the same as we can always obtain a version with reduced computational effort for one model by executing the other and translating the result.}. Moreover, they are interchangeable using the weight matrix definitions in Equations \eqref{eq:dRNN_Wh}-\eqref{eq:dRNN_o}.

\subsubsection{Relation to Other Topologies} 
\begin{figure*}[tb]
	\vskip 0.2in
\begin{center}
\centerline{\includegraphics[width=0.9\textwidth]{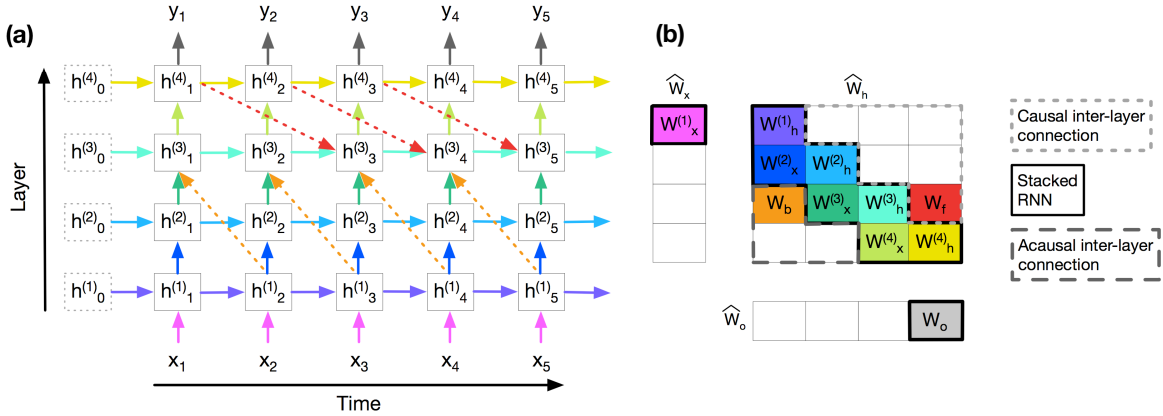}}
	\caption{A stacked RNN is equivalent to a single-layer d-RNN under the given sparse weight constraints. The d-RNN produces the same representations as the stacked network. (a) Stacked RNN with $k=4$ layers where connections show the different weight parameters. (b) Weights of the d-RNN that are equivalent to connections in the stacked RNN.}
	\label{fig:WeightMatrix}
\end{center}
	\vskip -0.2in
\end{figure*}
Suppose one takes a weight constrained d-RNN  and adds non-zero elements to regions not populated by weights in Eq.~\eqref{eq:dRNN_Wh}. These non-zero weights do not correspond to existing connections in the stacked RNN. So what do they correspond to?

To explore this question we illustrate a 4-layer stacked RNN in Figure~\ref{fig:WeightMatrix}~(a). Here, solid arrows show the standard stacked RNN connections. The d-RNN weight matrices \mahs{W}{h}, \mahs{W}{x}, and \mahs{W}{o} are shown in Figure~\ref{fig:WeightMatrix}~(b), where the color of each block matches the corresponding arrow in Figure~\ref{fig:WeightMatrix}~(a). 
Blocks on the main diagonal of \mahs{W}{h} connect groups of units to themselves recurrently, while blocks on the subdiagonal correspond to connections between layers in the stacked RNN. More generally, block $(i,j)$ in \mahs{W}{h} corresponds to a connection from $\vetl{h}{t}{j}$ to $\vetl{h}{t+j-i+1}{i}$ in the stacked RNN. Thus, blocks in the lower triangle (i.e.~$i>j+1$) correspond to connections that point backwards in time, and from a lower layer to a higher layer. For example, the orange block $(3,1)$ in Figure~\ref{fig:WeightMatrix}~(b) (and the dashed orange lines in Figure~\ref{fig:WeightMatrix}~(a)) connects layer 1 at time $t$ to layer 3 at time $t-1$. 
Conversely, blocks in the upper triangle (i.e.~$j>i$) point forward in time and from a higher layer to a lower layer. For example, the red block $(3,4)$ in Figure~\ref{fig:WeightMatrix}~(b) (and the dashed red lines in Figure~\ref{fig:WeightMatrix}~(a)) connects layer 4 at time $t$ to layer 3 at time $t+2$.

Thus we see that adding weights to empty regions in the weight constrained d-RNN can mimic the behavior of many stacked recurrent architectures that have previously been proposed. Among others, it can approximate the IndRNN \cite{IndRNN}, td-RNN \cite{tdRNN},  skip-connections \cite{stacking_graves}, and all-to-all layer networks \cite{All2All}. 
Simply removing the constraints on \mahs{W}{h} during training will enable a d-RNN to learn the necessary stacked architecture. However, unlike an ordinary RNN, this requires the output to be delayed based on the desired stacking depth. Further, while the single-layer network has the same total number of units as the corresponding stacked RNN, relaxing constraints on \mahs{W}{h} would mean that the single-layer would have many more parameters.

\subsection{Approximating Bidirectional RNNs}
\label{sec:ApproxBiRNN}
We previously showed how a d-RNN can be made equivalent to a stacked RNN by constraining its weight matrices. Without these constraints, the d-RNN has the ability to peek at ``future'' inputs: it computes the delayed output for time $t$ at \veht{y}{t+d} using also the inputs $\vet{x}{t+1},\dots,\vet{x}{t+d}$ that are beyond timestep $t$.
A similar idea was used in the past as a baseline for bidirectional recurrent neural networks (Bi-RNNs) \cite{BiRNN,BiLSTM}. These papers showed that Bi-RNNs were superior to d-RNNs for relatively simple problems, but it is not clear that this comparison holds true for problems that require more non-linear solutions. If a recurrent network can compute the output for time $t$ by exploiting future input elements, what conditions are necessary to approximate its Bi-RNN counterpart? Moreover, can the d-RNN obtain the same results? And, given these conditions, is there a benefit to using the d-RNN instead of the Bi-RNN?

Figure~\ref{fig:nonlinearities} shows the number of non-linear transformations that each network can apply to any input element before computing the output at timestep $t_0$. The generic RNN processes only past inputs ($t \le t_0$), and the number of non-linearities decreases for inputs closer to timestep $t_0$. The Bi-RNN has identical behavior for causal inputs but is augmented symmetrically for acausal inputs.
In contrast, the d-RNN has similar behavior for the causal inputs but with a higher number of non-linearities. This trend continues for the first $d$ acausal inputs with a decreasing number of non-linearities until the number reaches zero at $t=t_0+d+1$.
In order for a d-RNN to have at least as many non-linearities as a Bi-RNN for every element in a sequence, it would need a delay that is twice the sequence length. However, a d-RNN could beat a Bi-RNN when the non-linear influence of nearby acausal inputs on the learned function is larger than elements farther in the future. In these cases, stacking Bi-RNNs would be needed to achieve the same objective.

\begin{figure}[tb]
	\vskip 0.2in
	\begin{center}
	\centerline{\includegraphics[width=0.8\columnwidth]{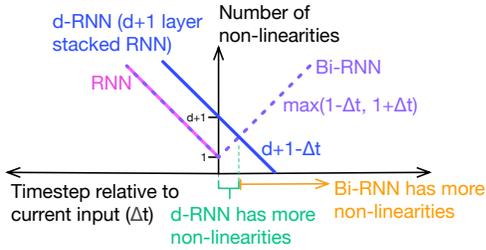}}
	\caption{Number of non-linearities that can be applied to past and future sequence elements with respect to current input ($\Delta t$=0). The d-RNN only sees $d$ steps into the future.}
	\label{fig:nonlinearities}
\end{center}
	\vskip -0.2in
\end{figure}

Using a d-RNN to approximate a Bi-RNN can also decrease computational cost. For a sequence of length $T$, a stacked Bi-RNN needs to compute both forward and backward RNNs for each layer before it can compute the next one. This synchronization requirement hinders parallelization and increases runtime. In contrast, the forward-pass for the d-RNN takes $T+d$ steps, but does not suffer from synchronization. Thus in highly parallel hardware such as CPUs and GPUs, the runtime of a $k$-layer stacked Bi-RNN should be at least $k$ times slower than an RNN or d-RNN.
Beyond computational costs, d-RNNs can also be used where it is critical to output values in (near) realtime applications~\cite{TimeSeriesAnomalyDetection,DeepVoice}. A d-RNN requires only the last $d$ elements and a hidden state to compute a new value, whereas bidirectional architectures need to process an entire backward pass of the sequence.

\section{Experiments}
We test the capabilities of the d-RNN in four experiments designed to shed more light on the relationships between d-RNNs, RNNs, Bi-RNNs, and stacked networks. For this purpose, the RNN implementation we use is a LSTM network, which avoids vanishing gradients and retains more information over long periods. The delayed LSTM networks are denoted as d-LSTMs. To train each d-LSTM, the input sequences are padded at the end with zero-vectors and loss is computed by ignoring the first ``delay'' timesteps, as explained in Section \ref{sec:dRNN}.
All models are trained using the Adam optimization algorithm \cite{ADAM} with learning rate $0.001$, $\beta_1=0.9$, and $\beta_2=0.999$.
During training, the gradients are clipped \cite{Pascanu_bengio2013} at 1.0 to avoid explosions.
Experiments were implemented using PyTorch 1.1.0 \cite{PyTorch}, and code can be found at \url{http://www.anonymous.com/anonymous}.

\subsection{Sequence Reversal}
First, we propose a simple test to illustrate how the d-LSTM can interpolate between a regular LSTM and Bi-LSTM.
In this test we require the recurrent architectures to output a sequence in reverse order while reading it, i.e.~$\vet{y}{t}=\vet{x}{T-t+1}$ for $t=1,..,T$.
Solving this task perfectly is only possible when a network has acausal access to the sequence. Moreover, depending on how many acausal elements a network can access, it is possible to  analytically calculate the expected maximum performance that the network can achieve. Given a sequence of length $T$ with elements from a vocabulary $\{1,...,V\}$, a causal network such as the regular LSTM can output the second half of the elements correctly and guess those in the first half with probability $\nicefrac{1}{V}$.
When a network has access to $d$ acausal elements it can start outputting correct elements before reaching the halfway point, and can achieve an expected true positive rate (TPR) of 
$\tfrac{1}{2} \left(1 + \tfrac{1}{V}\right) + \left\lfloor \tfrac{d+1}{2} \right\rfloor \tfrac{1}{T}\left(1 - \tfrac{1}{V}\right).$
We generate data sequences of length $T=20$ by uniformly sampling integer values between 1 and $V=4$. The training set consists of 10,000 sequences, the validation set 2,000, and test set 2,000. Output sequences are the input sequences reversed. Values in the input sequences are fed as one-hot vector representations. All networks output via a linear layer with a softmax function that converts to a vector of $V$ probabilities to which cross-entropy loss is applied. The LSTM and d-LSTM networks have 100 hidden units, while the Bi-LSTM has 70 in each direction in order to keep the total number of parameters constant. We use batches of 100 sequences and train for 1,000 epochs with early stopping after 10 epochs and $\Delta=$1e-3.

Figure \ref{fig:reverse} shows accuracy on this task as a function of the applied delay. The LSTM  does not use acausal information and is unable to reverse more than half of the input sequence. Conversely, the Bi-LSTM has full access to every element in the sequence, and can perfectly solve the task. For the d-LSTM network, performance increases as we increase the delay in the output, reaching the same level as the Bi-LSTM once the network has access to the entire sequence before being required to produce any output (delay 19). This experiment demonstrates that the d-LSTM can ``interpolate'' between LSTM and Bi-LSTM by choosing a delay that ranges between zero and the length of the input sequence.

\begin{figure}[tb]
	\vskip 0.2in
	\begin{center}
	\centerline{\includegraphics[trim=15 18 15 20, clip,width=0.95\columnwidth]{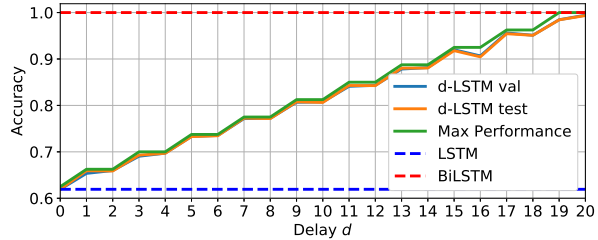}}
	\caption{Comparison of different delay values for a d-LSTM network for reversing a sequence. LSTM and Bi-LSTM networks are shown for reference. The network is capable of achieving the expected statistical bound. The d-LSTM with highest delay is capable of solving the task as well as the Bi-LSTM.}
	\label{fig:reverse}
	\end{center}
	\vskip -0.2in
\end{figure}

\subsection{Evaluating Network Capabilities}
\label{sec:PhasesExperiment}
\begin{figure*}[tb]
	\vskip 0.2in
	\begin{center}
	\includegraphics[trim=20 15 0 35, clip,width=\textwidth]{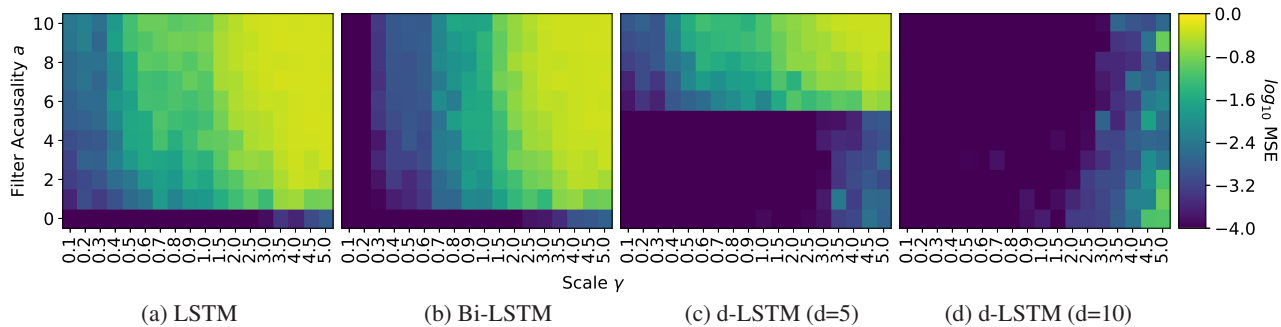}
	\footnotesize{
		\begin{tabular}{cccc}
			\hspace{0.5in}	(a) LSTM \hspace{0.5in} & \hspace{0.25in} (b) Bi-LSTM \hspace{0.25in} &\hspace{0.25in}(c) d-LSTM (d=5) \hspace{0.25in}&\hspace{0.25in}(d) d-LSTM (d=10) \hspace{0.5in}
		\end{tabular}
	}
	\caption{Error maps for the sine function experiment with different degrees of non-linearity (horizontal axis) and amounts of acausality of the filter (vertical axis). Tested architectures: (a) LSTM, (b) Bi-LSTM, (c) d-LSTM with delay=5, and (d) d-LSTM with delay=10. Dark blue regions depict perfect filtering (low error), transitioning to yellow regions with high error.}
	\label{fig:sinus}
	\end{center}
	\vskip -0.2in
\end{figure*}
The first experiment showed how a d-LSTM with sufficient delay can mimic a Bi-LSTM.
In the next experiment we aim at comparing how well d-LSTM, LSTM, and Bi-LSTM networks approximate functions with varying degrees of non-linearity and acausality.

Drawing inspiration from \cite{BiRNN}, we require each recurrent network to learn the function $\vet{y}{t}=\sin (\gamma \sum_{j=-c+1}^{a}\vet{w}{j+c} \vet{x}{t+j} )$, where \ve{w} is a linear filter. The parameter $\gamma$ scales the argument of the sine function and thus controls the degree of non-linearity in the function: for small $\gamma$ the function is roughly linear, while for large $\gamma$ the function is highly non-linear. Integers $a\ge 0 $ (acausal) and $c\ge 0$ (causal) control the length of the causal and acausal portions of the linear filter \ve{w} that is applied to the input \ve{x}. 

We generate datasets with different combinations of $\gamma \in [0.1,\dots, 5.0]$ and $a \in [0,\dots,10]$, choosing $c$ such that $a+c=20$. For each combination, we generate a filter \ve{w} with 20 elements drawn uniformly in $[0.0,1.0)$, and random input sequences with $T=50$ elements drawn from a uniform distribution $[0.0,1.0)$. In total, there are 10,000 generated sequences for training, 2,000 for validation, and 2,000 for testing with each set of parameter values.
The output is computed following the previous formula and with zero padding for the borders.
We generate 5 repetitions of each dataset with different filters \ve{w} and inputs \ve{x}.

We train LSTM, d-LSTM with delays 5 and 10, and Bi-LSTM networks to minimize mean squared error (MSE). The LSTM and d-LSTM have 100 hidden units and the Bi-LSTM has 70 per network, matching the numbers of parameters. A linear layer after the recurrent layer outputs a single value per timestep. Models are trained in batches of 100 sequences for 1,000 epochs. Training is stopped if the validation MSE falls below 1e-5. Training is repeated five times for each $\left(\gamma, a\right)$ value. 

Figure \ref{fig:sinus} shows the average test MSE for each model as a function of $\gamma$ (degree of input non-linearity) and $a$ (acausality). 
LSTM performance (Fig. \ref{fig:sinus}~(a)) is poor everywhere except where the filter is purely causal. Surprisingly, the network performs quite well even when the amount of non-linearity ($\gamma$) is quite high. The reason for this seems to be that temporal depth enables the LSTM to approximate this function well. Bi-LSTM performance (Fig. \ref{fig:sinus}~(b)) follows a similar trend for the causal case ($a=0$) as the forward LSTM, but also has good performance for acausal filters ($a>0$) when the function is nearly linear ($\gamma$ is small). As the non-linearity of the function increases, however, Bi-LSTM performance suffers. This occurs because the Bi-LSTM needs to approximate a highly non-linear function with a linear combination of its forward and backward outputs, which cannot be done with small error. Improving performance would require stacked Bi-LSTM layers.

In contrast, d-LSTM networks have excellent performance for both non-linear and acausal functions. The d-LSTM with delay 5 (Fig. \ref{fig:sinus}~(c)) shows a clear switch in performance from acausality $a=5$ to $6$.
This perfectly matches the limit of acausal elements that the network has access to.
For the d-LSTM with delay 10 (Fig. \ref{fig:sinus}~(d)), the network performs well for acausality values $a$ up to 10. 

An interesting outcome of this experiment is the better performance observed for the d-LSTM over the Bi-LSTM. This shows that the d-LSTM can be a better fit than a Bi-LSTM for the right task. Furthermore, the d-LSTM network seems to approximate the functionality of a stacked Bi-LSTM by approximating highly non-linear functions. In practice, this could be a great benefit for applications where there is no need to treat the whole sequence. Moreover, this could be impossible in other cases, such as streamed data. In such  cases, the d-LSTM would shine over bidirectional architectures. On the other hand, we expect the Bi-LSTM to perform better when the acausality needs for the task are longer than the delay, i.e., $a>d$.

\subsection{Masked Character-Level Language Modeling}
\begin{table*}[tb]
	\caption{Performance of different networks on the masked character-level language modeling task in bits per character (BPC); lower is better. Mean and standard deviation values are computed over 5 repetitions of training and inference runtime on the test set.}
	\label{tbl:mlm}
	\vskip 0.1in
	\begin{center}
		\begin{small}
			\begin{sc}
				\begin{tabular}{lccccccr}
					\toprule
					Model & Layers & Delay & Units / layer & Params. & Val. BPC & Test BPC & Runtime \\
					\midrule
					LSTM & 1 &  - & 1024 & 4271411 & $2.003 \pm 0.003$ & $\mathbf{2.075 \pm 0.002}$ & $ 3.44ms \pm 0.09$ \\
				        LSTM & 2 &  - &  594 & 4283641 & $2.015 \pm 0.005$ & $2.087 \pm 0.005$ & $ 4.93ms \pm 0.13$ \\
				        LSTM & 5 &  - &  343 & 4272372 & $2.091 \pm 0.016$ & $2.155 \pm 0.014$ & $17.22ms \pm 0.62$ \\
					\midrule
				     Bi-LSTM & 1 &  - &  722 & 4278879 & $0.977 \pm 0.004$ & $1.037 \pm 0.004$ & $ 4.97ms \pm 0.07$ \\
					Bi-LSTM & 2 &  - &  363 & 4277173 & $0.633 \pm 0.003$ & $\mathbf{0.677 \pm 0.002}$ & $13.72ms \pm 0.31$ \\
				     Bi-LSTM & 5 &  - &  202 & 4287151 & $0.637 \pm 0.003$ & $0.677 \pm 0.004$ & $29.18ms \pm 0.23$ \\
					\midrule
				      d-LSTM & 1 &  1 & 1024 & 4271411 & $1.332 \pm 0.001$ & $1.390 \pm 0.001$ & $ 3.29ms \pm 0.22 $ \\
				      d-LSTM & 1 &  5 & 1024 & 4271411 & $0.708 \pm 0.005$ & $0.755 \pm 0.004$ & $ 3.39ms \pm 0.08 $ \\
					d-LSTM & 1 &  8 & 1024 & 4271411 & $0.662 \pm 0.002$ & $\mathbf{0.706 \pm 0.003}$ & $ 3.36ms \pm 0.08 $ \\
				      d-LSTM & 1 & 10 & 1024 & 4271411 & $0.666 \pm 0.004$ & $0.709 \pm 0.004$ & $ 3.56ms \pm 0.10 $ \\
					\bottomrule
				\end{tabular}
			\end{sc}
		\end{small}
	\end{center}
	\vskip -0.1in
\end{table*}

Next we examined a language task which should benefit from acausal information, masked character-level language modeling. This task is adapted from previous work in training bidirectional language models \cite{devlin_bert_2019}. To generate masked sequences, we randomly replace each character with a mask token (`[MASK]') with 20\% probability. The task of the network is to predict the correct character when it encounters a mask token. Because each sequence contains multiple mask tokens, the network will need to fill in some mask tokens conditioned on an input sequence that already contains one or more mask tokens. This can be thought of as a signal reconstruction task: when sequential inputs are randomly degraded, how well can the network recover the true signal? Acausal information clearly helps with this reconstruction. For example, the missing letter in the sequence ``hik[MASK]ng'' is easier to predict than the sequence ``hik[MASK]''. 

We used \texttt{text8}, a clean 100MB sample of English Wikipedia text \cite{text8} which consists of 27 characters (the English alphabet and spaces). The input data contained an extra 28th mask character. These 28 characters were mapped to an input embedding layer of dimension 10. The output layer was independent of the input embedding, and only consisted of the 27 non-mask characters. Following previous work \cite{Mikolov_dataSplit}, the first 90M characters formed the training set, the next 5M the validation set, and the last 5M the test set. All models were trained with a sequence length of 180 characters, in mini-batches of 128 sequences for a total of 20 epochs. Success on the task is measured by calculating bits-per-character (BPC) for the mask tokens only. We measured forward-pass runtimes on a Nvidia Titan V GPU and report average time to process a mini-batch. 

The results are summarized in Table \ref{tbl:mlm}. As expected, the stacked Bi-LSTMs achieve the lowest BPC. However, as the number of layers increases, the inference runtime also increases because of the synchronization needed between layers. Notably, d-LSTMs with intermediate delays achieve a BPC that is within 5\% of the Bi-LSTM with at least $4\times$ faster runtime. Since all of the d-LSTMs have a single layer, inference runtime remains constant as the delay and the capacity of these networks increases. We find similar results for other network capacities (see supplementary material).

\subsection{Real-World Part-of-Speech Tagging}
\begin{table*}[tb]
	\caption{Parts-of-Speech performance for German, English, and French  languages. The models are composed of two subnetworks at character-level and word-level. Best bidirectional network and best forward-only network are marked in bold for each language.}
	\label{tbl:pos}
	\vskip 0.1in
	\begin{center}
		\begin{small}
			\begin{sc}
		\begin{tabular}{c|cccc}
			\toprule
			Language & Char-level network & Word-level network & Validation Accuracy & Test Accuracy \\
			\midrule
			&LSTM&LSTM &  $92.05 \pm 0.16$ &  $91.58 \pm 0.11$  \\
			German & d-LSTM delay=1&d-LSTM delay=1 &  $93.48 \pm 0.31$ &  $\mathbf{92.87 \pm 0.24}$  \\
			&d-LSTM  delay=1&Bi-LSTM &  $93.93 \pm 0.06$ &  $\mathbf{93.39 \pm 0.18}$  \\ 
			&Bi-LSTM&Bi-LSTM &  $93.88 \pm 0.13$ &  $93.15 \pm 0.08$  \\ 
			\midrule
			&LSTM&LSTM &  $92.05 \pm 0.13$ &  $92.14 \pm 0.10$  \\
			English &d-LSTM delay=1&d-LSTM  delay=1 &  $94.57 \pm 0.08$ &  $\mathbf{94.57 \pm 0.14}$  \\
			& d-LSTM delay=1&Bi-LSTM &  $94.94 \pm 0.07$ &  $\mathbf{94.95 \pm 0.06}$  \\ 
			& Bi-LSTM&Bi-LSTM &  $94.85 \pm 0.05$ &  $94.84 \pm 0.08$  \\
			\midrule
			& LSTM&LSTM &  $96.67 \pm 0.07$ &  $96.10 \pm 0.11$  \\
			French &d-LSTM with delay=1&d-LSTM with delay=1 &  $97.49 \pm 0.04$ &  $\mathbf{97.04 \pm 0.13}$  \\
			&d-LSTM with delay=1&Bi-LSTM & $97.67 \pm 0.07$ &  $\mathbf{97.23 \pm 0.12}$  \\ 
			& Bi-LSTM&Bi-LSTM &  $97.63 \pm 0.06$ &  $97.22 \pm 0.11$  \\ 
			\bottomrule
		\end{tabular}
\end{sc}
\end{small}
\end{center}
\vskip -0.1in
\end{table*}

In the previous experiments, we show that d-LSTM is capable of approximating and even outperforming a Bi-LSTM in some cases. In practice, however, the elements in a sequence may have different forward and backward relations. This poses a challenge for delayed networks that are constrained to a specific delay. If the delay is too low, it may not be enough for some long dependencies between elements. If it is too high, the network may forget information and require higher capacity (and maybe training data). 
This is prevalent in several NLP tasks. Therefore we compare the performance of the d-LSTM with a Bi-LSTM on an NLP task where Bi-LSTMs achieve state-of-the-art performance, the Part-of-Speech (POS) tagging task \cite{POSBiLSTM,POSBiLSTM2,POSDualBiLSTM}. The task involves processing a variable length sequence to predict a POS tag (e.g.~Noun, Verb) per word, using the Universal Dependencies (UD)~\cite{UD} dataset. More details can be   found in the supplementary material.

The dual Bi-LSTM architecture proposed by \citet{POSDualBiLSTM} is followed to test the approximation capacity of the d-LSTMs.
In this model, a word is encoded using a combination of word embeddings and character-level encoding. The encoded word is fed to a Bi-LSTM followed by a linear layer with softmax to produce POS tags. The character-level encoding is produced by first computing the embedding of each character and then feeding it to a Bi-LSTM. The last hidden state in each direction is concatenated with the word embedding to form the character-level encoding.

The character-level Bi-LSTM has 100 units in each direction and the LSTM/d-LSTMs have 200 units to generate encodings of the same size. For the word-level subnetwork, the hidden state is of size 188 for the Bi-LSTM, and 300 units for the LSTM/d-LSTM to match the number of parameters. The networks are trained for 20 epochs with cross-entropy loss. We train combinations of networks with delays 0 (LSTM), 1, 3, and 5 for the character-level subnetwork, and delays 0 through 4 for the word-level. Each network has 5 repeats with random initialization.

Results are presented in Table \ref{tbl:pos}. For brevity, we include a subset of the combinations for each language (the complete table can be found in the supplementary material). For the character-level model, LSTMs without delay yield reduced performance. However, replacing only the character-level Bi-LSTM with a LSTM does not affect the performance (supplementary material). This suggests that only the word-level subnetwork benefits from acausal elements in the sentence. 
Interestingly, using a d-LSTM with delay 1 for the character-level network achieves a small improvement over the double-bidirectional model in English and German. 
Replacing the word-level Bi-LSTM with an LSTM decreases performance significantly. However, using even a d-LSTM with delay 1 improves performance to within 0.3\% of the original Bi-LSTM model.

\section{Conclusions}
In this paper we analyze the d-RNN, a single layer RNN where the output is delayed relative to the input. We show that this simple modification to the classical RNN adds both depth in time and acausal processing. We prove that a d-RNN is a superset of stacked RNNs, which are frequently used for sequence problems: a d-RNN with output delay $d$ and specific constraints on its weights is exactly equivalent to a stacked RNN with $d+1$ layers.
We also show that the d-RNN can approximate bidirectional RNNs and stacked bidirectional RNNs because the delay allows the model to look at future as well as past inputs.
In sum, we found that d-RNNs are a simple, elegant, and computationally efficient alternative that captures many of the best features of different RNN architectures while avoiding many downsides.


\bibliography{references}
\bibliographystyle{icml2020}

\ifdefined\preprint		

\appendix
\renewcommand{\theequation}{\thesection.\arabic{equation}}

\section{Theorem 1 Proof}
\label{sec:proof-thm1}

Let us recall the notation introduced in the main paper.
We use superscript $(i)$ to refer to a weight matrix or vector related to layer $i$ in a stacked network, e.g., \masl{W}{h}{i}, or \vetl{h}{t}{i}. For a single-layer d-RNN, we refer to weight matrices and related vectors with "hat", e.g., \mahs{W}{h} or \veht{h}{t}. Additionally, we define the block notation as subvector \vehtb{v}{t}{i} refers to the $i$-th block of vector \veht{v}{t} composed of $k$ blocks. The blocks follow the definition in Equations (3)-(5).

\begin{proof}[Proof of Theorem 1]
We prove Theorem 1 
by induction on the sequence length $t$. First, we show that for $t=1$ the stacked RNN and the d-RNN with the constrained weights are equivalent. Namely, for $t=1$ we show that the outputs and the hidden states are the same, i.e. $\veht{y}{k} = \vet{y}{1}$ and $\vehtb{h}{i}{i}  = \vetl{h}{1}{i}$, respectively.
Without loss of generality, we have  for any $i$ in $1 \dots k$ the following:
\begin{align*}
	\vehtb{h}{i}{i} &= \fb{i}{\mahs{W}{x} \vet{x}{i} + \mahs{W}{h} \veht{h}{i-1} + \vehs{b}{h}} \\ 
	& =  \f{\mahsb{W}{x}{i} \vet{x}{i}  + \masl{W}{x}{i}\vehtb{h}{i-1}{i-1} + \masl{W}{h}{i} \vehtb{h}{i-1}{i} + \vesl{b}{h}{i}} \\
	& =  \f{\ve{0} + \masl{W}{x}{i}\vehtb{h}{i-1}{i-1} + \masl{W}{h}{i} \vetl{h}{0}{i} + \vesl{b}{h}{i}} \\
	& =  \fopen{\masl{W}{x}{i} \cdot} \\
	&\qquad	\f{\masl{W}{x}{i-1}\vehtb{h}{i-2}{i-2} + \masl{W}{h}{i-1} \vehtb{h}{i-2}{i-1} + \vesl{b}{h}{i-1}} \\
		&\quad \fclose{+ \masl{W}{h}{i} \vetl{h}{0}{i} + \vesl{b}{h}{i}} \\
	& =  \fopen{\masl{W}{x}{i} \cdot } \\
	& \qquad	\f{\masl{W}{x}{i-1}\vehtb{h}{i-2}{i-2} + \masl{W}{h}{i-1} \vetl{h}{0}{i-1} + \vesl{b}{h}{i-1}} \\
	& \quad	\fclose{+ \masl{W}{h}{i} \vetl{h}{0}{i} + \vesl{b}{h}{i}} \\
	& = \dots \\
	& = \fopen{\masl{W}{x}{i} \ldots \fopen{\masl{W}{x}{j}  \ldots}} \\
			&\qquad \fopen{\masl{W}{x}{2}\cdot \f{\masl{W}{x}{1} \vet{x}{1} + \masl{W}{h}{1} \vetl{h}{0}{1} + \vesl{b}{h}{1}}} \\
			& \qquad \fclose{+ \masl{W}{h}{2} \vetl{h}{0}{2} + \vesl{b}{h}{2}} \\
		&\quad\fclose{\fclose{ \ldots + \masl{W}{h}{j} \vetl{h}{0}{j} + \vesl{b}{h}{j}}
			\ldots + \masl{W}{h}{i} \vetl{h}{0}{i} + \vesl{b}{h}{i}} 
\end{align*}
\begin{align*}
	& =  \fopen{\masl{W}{x}{i} \ldots \fopen{\masl{W}{x}{j}  \ldots} } \\
			& \qquad \f{\masl{W}{x}{2}\vetl{h}{1}{1}	+ \masl{W}{h}{2} \vetl{h}{0}{2} + \vesl{b}{h}{2}} \\
		& \quad \fclose{\fclose{
			\ldots + \masl{W}{h}{j} \vetl{h}{0}{j} + \vesl{b}{h}{j}}
		\ldots + \masl{W}{h}{i} \vetl{h}{0}{i} + \vesl{b}{h}{i}} \\
	& = \dots \\
	& =  \fopen {\masl{W}{x}{i} \ldots }\\
	&\qquad \f{\masl{W}{x}{j}\vetl{h}{1}{j-1}+ \masl{W}{h}{j} \vetl{h}{0}{j} + \vesl{b}{h}{j}}\\
		& \quad \fclose{ \ldots + \masl{W}{h}{i} \vetl{h}{0}{i} + \vesl{b}{h}{i}} \\
	& = \dots \\
	& = \f{\masl{W}{x}{i} \vetl{h}{1}{i-1} + \masl{W}{h}{i} \vetl{h}{0}{i} + \vesl{b}{h}{i}} \\
	& = \vetl{h}{1}{i},
\end{align*}
where we used the initialization assumption $\vehtb{h}{i-1}{i}=\vetl{h}{0}{i}$ for all $i=1\dots k$, and the definition of the hidden state in Equations (3)-(4) 
for $j-1$ blocks, in the previous steps. In particular, we have for $j=k$,
\begin{equation*}
\vehtb{h}{k}{k}  =  \vetl{h}{1}{k}.
\end{equation*}
Plugging this result and the definition of the output weights and biases in Equation (8) 
into Equation (2) 
for computing the output, we obtain
\begin{align}
\veht{y}{k} &= \g{\mahs{W}{o} \veht{h}{k}  + \vehs{b}{o}} \nonumber\\
&=  \g{\mas{W}{o} \vehtb{h}{k}{k}  + \ves{b}{o}} \nonumber \\
&= \g{\mas{W}{o} \vetl{h}{1}{k}  + \ves{b}{o}} \nonumber \\
&= \vet{y}{1}.
\label{eq:dRNN_proof_yk}
\end{align}
Which concludes the basis of the induction.

Next, we assume that $\vehtb{h}{t+i-1}{i}  = \vetl{h}{t}{i}$ for all $1 \le i \le k $ and $t \le T-1$, and prove that it holds for the hidden states for all layers when $t=T$: $\vehtb{h}{T+i-1}{i}  = \vetl{h}{T}{i}$, $\forall 1\le i \le k$. Without loss of generality, we have for the hidden state \vehtb{h}{T+i-1}{i} in constrained weights single-layer d-RNN that,
\begin{align*}
	\vehtb{h}{T+i-1}{i} &= \fb{i}{\mahs{W}{x} \vet{x}{T+i-1} + \mahs{W}{h} \veht{h}{T+i-2} + \vehs{b}{h}} \\ 
	& =  \fopen{\mahsb{W}{x}{i} \vet{x}{T+i-1}  + \masl{W}{x}{i}\vehtb{h}{T+i-2}{i-1}} \\
	& \quad \fclose{+ \masl{W}{h}{i} \vehtb{h}{T+i-2}{i} + \vesl{b}{h}{i}} \\
	& =  \f{\ve{0} + \masl{W}{x}{i}\vehtb{h}{T+i-2}{i-1} + \masl{W}{h}{i}\vehtb{h}{T+i-2}{i}  + \vesl{b}{h}{i}} \\
	& =  \fopen{\masl{W}{x}{i} \cdot \fopen{\masl{W}{x}{i-1}\vehtb{h}{T+i-3}{i-2} + } } \\
	& \qquad \fclose{\masl{W}{h}{i-1} \vehtb{h}{T+i-3}{i-1}+ \vesl{b}{h}{i-1}} \\
	& \quad \fclose{+ \masl{W}{h}{i} \vehtb{h}{T+i-2}{i}  + \vesl{b}{h}{i}} \\
	& = \dots \\
	& =  \fopen{\masl{W}{x}{i} \ldots \fopen{\masl{W}{x}{j}  \ldots }} \\
	&  \qquad \fopen{\masl{W}{x}{2} \f{\masl{W}{x}{1} \vet{x}{T} + \masl{W}{h}{1} \vehtb{h}{T-1}{1} + \vesl{b}{h}{1}} } \\
	& \qquad \fclose{+ \masl{W}{h}{2} \vehtb{h}{T}{2} + \vesl{b}{h}{2}} \\
	& \enspace \quad  \fclose{ \ldots + \masl{W}{h}{j}  \vehtb{h}{T+j-2}{j} + \vesl{b}{h}{j}} \\
	& \quad \fclose{ \ldots + \masl{W}{h}{i} \vehtb{h}{T+i-2}{i}  + \vesl{b}{h}{i}}
\end{align*}
From the inductive assumption we have $\vehtb{h}{T+j-2}{j} = \vetl{h}{T-1}{j} $ for all $1 \le j \le k$, then it follows 
\begin{align*}
	\vehtb{h}{T+i-1}{i}	& =  \fopen{\masl{W}{x}{i} \ldots 
		\fopen{\masl{W}{x}{j}  \ldots } } \\
	& \qquad \fopen{\masl{W}{x}{2} \f{\masl{W}{x}{1} \vet{x}{T} + \masl{W}{h}{1}\vetl{h}{T-1}{1} + \vesl{b}{h}{1}} } \\
	& \qquad \fclose{ + \masl{W}{h}{2} \vetl{h}{T-1}{2}  + \vesl{b}{h}{2}} \\
	& \enspace \quad \fclose{ \ldots + \masl{W}{h}{j}  \vetl{h}{T-1}{j} + \vesl{b}{h}{j}} \\
	& \quad \fclose{ \ldots + \masl{W}{h}{i} \vetl{h}{T-1}{i} + \vesl{b}{h}{i}}\\
	& =  \fopen{\masl{W}{x}{i} \ldots \fopen{\masl{W}{x}{j}  \ldots }} \\
	& \qquad \f{\masl{W}{x}{2}\vetl{h}{T}{1} + \masl{W}{h}{2} \vetl{h}{T-1}{2} + \vesl{b}{h}{2}} \\
	& \enspace \quad \fclose{ \ldots + \masl{W}{h}{j}  \vetl{h}{T-1}{j} + \vesl{b}{h}{j}} \\
	& \quad \fclose{ \ldots + \masl{W}{h}{i} \vetl{h}{T-1}{i} + \vesl{b}{h}{i}}\\
	& = \dots \\
	& =  \fopen{\masl{W}{x}{i} \ldots} \\
	& \qquad \f{\masl{W}{x}{j}\vetl{h}{T}{j-1}+ \masl{W}{h}{j} \vetl{h}{T-1}{j} + \vesl{b}{h}{j}} \\
	& \quad	\fclose{ \ldots + \masl{W}{h}{i} \vetl{h}{0}{i} + \vesl{b}{h}{i}} 
\end{align*}
\begin{align*}
	& = \dots \\
	& =  \f{\masl{W}{x}{i} \vetl{h}{T}{i-1} + \masl{W}{h}{i} \vetl{h}{T-1}{i} + \vesl{b}{h}{i}} \\
	& =  \vetl{h}{T}{i},
\end{align*}
where we used the definition of the hidden states in Equations (3)-(4). 
In particular, we have for $i=k$ that $\vehtb{h}{T+k-1}{k}=\vetl{h}{T}{k}$.

Now, we show that $\veht{y}{T+k-1}=\vet{y}{T}$.
By the definition of the output weights and biases in Equation (8). 
and by the fact that $\vehtb{h}{T+k-1}{k}=\vetl{h}{T}{k}$, we obtain
\begin{align}
\veht{y}{T+k-1} &= \g{\mahs{W}{o} \veht{h}{T+k-1}  + \vehs{b}{o}} \nonumber\\
&=  \g{\mas{W}{o} \vehtb{h}{T+k-1}{k}  + \ves{b}{o}} \nonumber\\
&= \g{\mas{W}{o} \vetl{h}{T}{k} + \ves{b}{o}} \nonumber \\
&= \vet{y}{T}, \label{eq:proof_steps_y}
\end{align}
which completes the proof.
\end{proof}

\section{Lemma 1 Proof}
\label{sec:proof-InitializationLemma}
We show next that there exists an initialization vector that allows us to initialize the equivalent single-layer weight constrained d-RNN as defined in Theorem 1. 
\begin{proof}[Proof of Lemma 1]
From the surjective definition of the activation function \f{\cdot}, we know that the function \f{\cdot} is right-invertible.
Namely, there is a function $\mathrm{r}:D\rightarrow \real{}$ such that for any $d\in D$, \rf{\cdot} satisties $\f{\rf{d}}=d$.
First, we note that for $i=1$, we have $\vehtb{h}{0}{1}=\vetl{h}{0}{1}$. When $i=2$, we have
\begin{align}
\vetl{h}{0}{2} &=  \vehtb{h}{1}{2} &= \f{ \masl{W}{x}{2} \vetl{h}{0}{1} + \masl{W}{h}{2}\vehtb{h}{0}{2} + \vesl{b}{h}{2} }. \label{eq:inith_0_2}
\end{align}
From \eqref{eq:inith_0_2} and the right-invertible function \rf{\cdot} satisfies $\vetl{h}{0}{2} = \f{\rf{\vetl{h}{0}{2}}}$, we obtain
\begin{align}
\rf{\vetl{h}{0}{2}} =&  \masl{W}{x}{2} \vetl{h}{0}{1} + \masl{W}{h}{2}\vehtb{h}{0}{2} + \vesl{b}{h}{2} \nonumber \\
\Longrightarrow \vehtb{h}{0}{2} =&   {\masl{W}{h}{2}}^{\dagger}\left[\rf{\vetl{h}{0}{2}} - \masl{W}{x}{2} \vetl{h}{0}{1} - \vesl{b}{h}{2}\right],
\label{eq:initrh_0_2}
\end{align}
where $\ma{A}^{\dagger}$ is the pseudoinverse of matrix \ma{A}.

We assume that we obtained the initializations for $i-1$ and compute the initialization for block $i$.In general, for block $i$ we have 
\begin{eqnarray}
\vetl{h}{0}{i}& =& \vehtb{h}{i-1}{i} = \f{ \masl{W}{x}{i} \vehtb{h}{i-1}{i-1} + \masl{W}{h}{i}\vehtb{h}{i-2}{i} + \vesl{b}{h}{i} } \nonumber 
\end{eqnarray}
We can plug in the initialization and the intermediate computed hidden states for block $i-1$ to obtain
\begin{eqnarray}
\vehtb{h}{i-2}{i} & =&  {\masl{W}{h}{i}}^{\dagger}\left[\rf{\vetl{h}{0}{i}} - \masl{W}{x}{i} \vehtb{h}{i-1}{i-1} - \vesl{b}{h}{i}\right].   \nonumber 
\end{eqnarray}
We continue to reapply the recursive formula one step at a time until we reach the last step before the initialization \vehtb{h}{0}{i}:
\begin{align}
\vehtb{h}{i-j}{i}  = &  {\masl{W}{h}{i}}^{\dagger}\left[\rf{\vehtb{h}{i-j+1}{i}} - \masl{W}{x}{i} \vehtb{h}{i-j+1}{i-1} - \vesl{b}{h}{i}\right]  \nonumber \\
 \vdots & \nonumber \\
\vehtb{h}{1}{i} =& \f{ \masl{W}{x}{i} \vehtb{h}{1}{i-1} + \masl{W}{h}{i}\vehtb{h}{0}{i} + \vesl{b}{h}{i} } \nonumber \\
\Longrightarrow \vehtb{h}{0}{i} =&   {\masl{W}{h}{i}}^{\dagger}\left[\rf{\vehtb{h}{1}{i}} - \masl{W}{x}{i} \vehtb{h}{1}{i-1} - \vesl{b}{h}{i}\right],
\label{eq:initrh_0_i}
\end{align}
Following these steps from \vetl{h}{0}{i} to obtain \vehtb{h}{0}{i}, we constructed the initialization of the weight constrained d-RNN to accurately mimic the initialization of the stacked RNN.
\end{proof}

\section{Extension to d-LSTMs}
A Long Short-Term Memory recurrent cell \cite{LSTM} is given by the introduction of a cell state and a series of gates that control the updates of the states. The cell state together with the gates aim to solve the vanishing gradients problems in the RNN. 
The LSTM cell is highly popular and we refer to the following implementation:
\begin{align}
\veht{e}{t} &= \fsigmoid{\mahs{W}{xe}\vet{x}{t} + \mahs{W}{he}\veht{h}{t-1} + \vehs{b}{e} }, \label{eq:lstm-inputgate}\\
\veht{f}{t} &= \fsigmoid{\mahs{W}{xf}\vet{x}{t} + \mahs{W}{hf}\veht{h}{t-1} + \vehs{b}{f}}, \label{eq:lstm-forgetgate}\\
\veht{o}{t} &= \fsigmoid{\mahs{W}{xo}\vet{x}{t} + \mahs{W}{ho}\veht{h}{t-1} + \vehs{b}{o}}, \label{eq:lstm-outputgate}\\
\veht{g}{t} &= \ftanh{\mahs{W}{xc}\vet{x}{t} + \mahs{W}{hc}\veht{h}{t-1} + \vehs{b}{c}}, \label{eq:lstm-cellgate}\\
\veht{c}{t} &= \veht{f}{t} \odot \veht{c}{t-1} + \veht{e}{t} \odot \veht{g}{t}, \label{eq:lstm-cellstate}\\
\veht{h}{t} &= \veht{o}{t}  \odot \ftanh{\veht{c}{t}}, \label{eq:lstm-hiddenstate}
\end{align}
where \veht{e}{t} is the input gate, \veht{f}{t} the forget gate, \veht{o}{t} the output gate, \veht{g}{t} the cell gate, \veht{c}{t} the cell state, and \veht{h}{t} the hidden state. The weight matrices are symbolized \mahs{W}{xa} and \mahs{W}{ha} as well as the bias \vehs{b}{a}, with $\mathbf{a} \in  \{\ve{e},\ve{c},\ve{f},\ve{o}\}$ being the respective gate. The symbol $\odot$ represents an element-wise product and \fsigmoid{\cdot} is the sigmoid function.

First, we note that the set of Equations \eqref{eq:lstm-inputgate}-\eqref{eq:lstm-hiddenstate} can be expanded into the following two equations:
\begin{align}
\veht{c}{t} = &  \fsigmoid{\mahs{W}{xf}\vet{x}{t} + \mahs{W}{hf}\veht{h}{t-1} + \vehs{b}{f}} \odot \veht{c}{t-1} \nonumber \\
&+  \fsigmoid{\mahs{W}{xe}\vet{x}{t} + \mahs{W}{he}\veht{h}{t-1} + \vehs{b}{e} } \nonumber \\
& \odot  \ftanh{\mahs{W}{xc}\vet{x}{t} + \mahs{W}{hc}\veht{h}{t-1} + \vehs{b}{c}}, \label{eq:lstm-cellstate-short}\\
\veht{h}{t} = & \fsigmoid{\mahs{W}{xo}\vet{x}{t} + \mahs{W}{ho}\veht{h}{t-1} + \vehs{b}{o}}  \odot \ftanh{\veht{c}{t}}. \label{eq:lstm-hiddenstate-short}
\end{align}
Rewriting the LSTM Equations  \eqref{eq:lstm-inputgate}-\eqref{eq:lstm-hiddenstate} in this form, allows to remain with the recurrent equations where both \veht{h}{t} and \veht{c}{t} depend on the previous hidden and cell states, \veht{h}{t-1} and \veht{c}{t-1} , and the current input \vet{x}{t}.

Next, we describe the weight matrices for the single-layer d-LSTM that matches a stacked-LSTM with $k$ layers. The matrices and biases follow the exact same pattern as the RNN proof, being the same for all gates. 
\begin{align}
\mahs{W}{ha}  &=   \begin{bmatrix} 
\masl{W}{ha}{1}  & \ma{0} & & \cdots& & \ma{0} \\ 
\masl{W}{xa}{2}  & \masl{W}{ha}{2}  & & & & \\
\ma{0}& \ddots & \ddots & \ddots & & \vdots\\
\vdots& \ddots & \masl{W}{xa}{i}  & \masl{W}{ha}{i} & \ddots & \\
& & & \ddots & \ddots& \ma{0}\\
\ma{0}& \cdots& & \ma{0} & \masl{W}{xa}{k}  & \masl{W}{ha}{k}  
\end{bmatrix} \label{eq:dLSTM_Wh} \\ 
\vehs{b}{ha} & =  \begin{bmatrix} 
\vesl{b}{ha}{1} \\ 
\vdots \\
\vesl{b}{ha}{k}
\end{bmatrix}, \qquad \qquad 
\mahs{W}{xa}  =  \begin{bmatrix} 
\masl{W}{xa}{1} \\ 
\ma{0} \\
\vdots \\
\ma{0}
\end{bmatrix}, \label{eq:dLSTM-all} 
\end{align}
where $\mahs{W}{xa}\in\real{kn \times q}$ are the input weights, $\mahs{W}{ha}\in\real{kn \times kn}$ the recurrent weights,  $\vehs{b}{ha}\in\real{kn}$ the biases, for gate $\ve{a} \in \{\ve{e}, \ve{c}, \ve{o}, \ve{f}\}$. We follow the same notation for blocks and layers introduced with Theorem 1.
We omit the equations for the output element \veht{y}{t} as they are exactly the same as the RNN in Theorem 1, and thus require the same steps for proving that outputs are equal, i.e., $\veht{y}{T+k-1}=\vet{y}{T}$. Therefore, for the LSTM theorem we will focus on the hidden and cell states.
\begin{theorem}
	\label{thm:dLSTM}
	Given an input sequence $\set{\vet{x}{t}}_{t=1...T}$ and a stacked LSTM with $k$ layers, and initial states $\{\vetl{h}{0}{i}, \vetl{c}{0}{i}\}_{i=1\dots k}$, the d-LSTM with delay $d=k-1$, defined by Equations \eqref{eq:dLSTM_Wh}-\eqref{eq:dLSTM-all} and initialized with \veht{h}{0} such that $\vehtb{h}{i-1}{i}=\vetl{h}{0}{i}, \; \forall i=1\dots k$ and \veht{c}{0} such that $\vehtb{c}{i-1}{i}=\vetl{c}{0}{i}, \; \forall i=1\dots k$, produces the same output sequence but delayed by $k-1$ timesteps, i.e., $\veht{y}{t+k-1} = \vet{y}{t}$ for all $t=1\dots T$. Further, the sequence of hidden and cell states at each layer $i$ are equivalent with delay $i-1$, i.e., $\vehtb{h}{t+i-1}{i}  = \vetl{h}{t}{i}$ and $\vehtb{c}{t+i-1}{i}  = \vetl{c}{t}{i}$ for all $1 \le i \le k $ and $t\ge1$.
\end{theorem}
\begin{proof}
We prove Theorem \ref{thm:dLSTM} by induction on the sequence length $t$. First, we show that for $t=1$ the stacked LSTM and the d-LSTM with the constrained weights are equivalent. Namely, for $t=1$ we show that the outputs, hidden states and cell states are the same, i.e. $\veht{y}{k} = \vet{y}{1}$, $\vehtb{h}{i}{i}  = \vetl{h}{1}{i}$, and $\vehtb{c}{i}{i}  = \vetl{c}{1}{i}$, respectively.
Without loss of generality, we have  for any $j$ in $1 \dots k$ the following:
\begin{align*}
	\vehtb{h}{i}{i} &= \fsigmoid{\mahsb{W}{xo}{i}\vet{x}{i} + \mahsb{W}{ho}{i}\vehtb{h}{i-1}{i} + \vehsb{b}{o}{i}}  \odot \ftanh{\vehtb{c}{i}{i}} \\
	& = \fsigmoid{\masl{W}{xo}{i}\vehtb{h}{i-1}{i-1} +\masl{W}{ho}{i}\vehtb{h}{i-1}{i} + \vesl{b}{o}{i}} \\
	&\quad  \odot \ftanh{\vehtb{c}{i}{i}} \\
	& = \fsigmoid{\masl{W}{xo}{i}\vehtb{h}{i-1}{i-1} +\masl{W}{ho}{i}\vetl{h}{0}{i} + \vesl{b}{o}{i}}  \\ 
	& \quad \odot \ftanhopen{
		\fsigmoid{\masl{W}{xf}{i}\vehtb{h}{i-1}{i-1} +\masl{W}{hf}{i}\vetl{h}{0}{i} + \vesl{b}{f}{i}} } \\
	& \quad \odot 	\vehtb{c}{i-1}{i}  \\
    & \quad	+\fsigmoid{\masl{W}{xe}{i}\vehtb{h}{i-1}{i-1} +\masl{W}{he}{i}\vetl{h}{0}{i} + \vesl{b}{e}{i}} \\ 
    & \quad \ftanhclose{\odot \ftanh{\masl{W}{xc}{i}\vehtb{h}{i-1}{i-1} +\masl{W}{hc}{i}\vetl{h}{0}{i} + \vesl{b}{c}{i}}	} \\
    & = \dots  \\
	& = \fsigopen{\masl{W}{xo}{i} \dots \left\{	\fsigopen{\masl{W}{xo}{j} \left[ \dots \right.  } \right. } \\
	& \; \qquad \fsigopen{\masl{W}{xo}{2} \fsigmoid{\masl{W}{xo}{1}\vet{x}{1} + \masl{W}{ho}{1}\vetl{h}{0}{1} + \vesl{b}{o}{1}} } \\
    & \enspace \qquad \odot \ftanhopen{ \fsigmoid{\masl{W}{xf}{1}\vet{x}{1} +\masl{W}{hf}{1}\vetl{h}{0}{1} + \vesl{b}{f}{1}} } \\
   	& \enspace \qquad \odot \vetl{c}{0}{1}  \\
    & \enspace \qquad + \fsigmoid{\masl{W}{xe}{1}\vet{x}{1} +\masl{W}{he}{1}\vetl{h}{0}{1} + \vesl{b}{e}{1}} \\
    & \enspace \qquad \ftanhclose{ \odot \ftanh{\masl{W}{xc}{1}\vet{x}{1} +\masl{W}{hc}{1}\vetl{h}{0}{1} + \vesl{b}{c}{1}} } \\
    & \; \qquad \fsigclose{+ \masl{W}{ho}{2}\vetl{h}{0}{2} + \vesl{b}{o}{2}} \\
	& \; \qquad \odot \ftanhopen{ \fsigmoid{\masl{W}{xf}{2}\left(\dots \right)  +\masl{W}{hf}{2}\vetl{h}{0}{2} + \vesl{b}{f}{2}} } \\
	& \; \qquad  \odot \vehtb{c}{1}{2} \\
	& \; \qquad + \fsigmoid{\masl{W}{xe}{2}\left(\dots\right) +\masl{W}{he}{2}\vetl{h}{0}{2} + \vesl{b}{e}{2}} \\
	& \; \qquad \ftanhclose{\odot \ftanh{ \masl{W}{xc}{2}\left(\dots\right) +\masl{W}{hc}{2}\vetl{h}{0}{2} + \vesl{b}{c}{2}}}  \\
	& \enspace \quad \fsigclose{ \left. \dots \right] +\masl{W}{ho}{j}\vetl{h}{0}{j} + \vesl{b}{o}{j}}  \\
	& \enspace \quad \odot \ftanhopen{ \fsigmoid{\masl{W}{xf}{j} \left[\dots \right] +\masl{W}{hf}{j}\vetl{h}{0}{j} + \vesl{b}{f}{j}} } \\
	& \enspace \quad \odot \vehtb{c}{j-1}{j} \\ 
	& \enspace \quad + \fsigmoid{\masl{W}{xe}{j} \left[\dots \right]  +\masl{W}{he}{j}\vetl{h}{0}{j} + \vesl{b}{e}{j}} 
\end{align*}
\begin{align*}
	& \enspace \quad  \left. \ftanhclose{ \odot \ftanh{\masl{W}{xc}{j} \left[\dots \right]  +\masl{W}{hc}{j}\vetl{h}{0}{j} + \vesl{b}{c}{j}} }  \right\}  \\
    & \quad \dots  +\masl{W}{ho}{i}\vetl{h}{0}{i} + \vesl{b}{o}{i}   \\
	& \quad \odot \ftanhopen{ \fsigmoid{\masl{W}{xf}{i} \left( \dots \right) +\masl{W}{hf}{i}\vetl{h}{0}{i} + \vesl{b}{f}{i}} } \\
	& \quad \odot \vehtb{c}{i-1}{i} \\
	& \quad	+ \fsigmoid{\masl{W}{xe}{i} \left( \dots\right)  +\masl{W}{he}{i}\vetl{h}{0}{i} + \vesl{b}{e}{i}} \\
	& \quad \ftanhclose{ \odot \ftanh{\masl{W}{xc}{i} \left(\dots \right)  +\masl{W}{hc}{i}\vetl{h}{0}{i} + \vesl{b}{c}{i}}} \\
	& = \fsigopen{\masl{W}{xo}{i} \dots \left\{ \fsigopen{\masl{W}{xo}{j} \left[\dots  \right.  } \right. } \\
	& \; \qquad \fsigmoid{\masl{W}{xo}{2}\vetl{h}{1}{1} +  \masl{W}{ho}{2}\vetl{h}{0}{2} + \vesl{b}{o}{2}} \\
	& \; \qquad \odot \ftanhopen{ \fsigmoid{\masl{W}{xf}{2}\vetl{h}{1}{1} +\masl{W}{hf}{2}\vetl{h}{0}{2} + \vesl{b}{f}{2}} \odot 	\vetl{c}{0}{2}} \\
	& \; \qquad + \fsigmoid{\masl{W}{xe}{2}\vetl{h}{1}{1} +\masl{W}{he}{2}\vetl{h}{0}{2} + \vesl{b}{e}{2}} \\
	& \; \qquad \ftanhclose{\odot \ftanh{ \masl{W}{xc}{2}\vetl{h}{1}{1} +\masl{W}{hc}{2}\vetl{h}{0}{2} + \vesl{b}{c}{2}}}  \\
	& \enspace \quad \fsigclose{ \left. \dots \right] +\masl{W}{ho}{j}\vetl{h}{0}{j} + \vesl{b}{o}{j}}  \\
	& \enspace \quad \odot \ftanhopen{ \fsigmoid{\masl{W}{xf}{j} \left[\dots \right] +\masl{W}{hf}{j}\vetl{h}{0}{j} + \vesl{b}{f}{j}}  \odot  \vetl{c}{0}{j}} \\ 
	& \enspace \quad + \fsigmoid{\masl{W}{xe}{j} \left[\dots \right]  +\masl{W}{he}{j}\vetl{h}{0}{j} + \vesl{b}{e}{j}} \\
	& \enspace \quad \left. \ftanhclose{ \odot \ftanh{\masl{W}{xc}{j} \left[\dots \right]  +\masl{W}{hc}{j} \vetl{h}{0}{j} + \vesl{b}{c}{j}}	} \right\}  \\
	& \quad \dots +\masl{W}{ho}{i}\vetl{h}{0}{i} + \vesl{b}{o}{i}  \\
	& \quad \odot \ftanhopen{ \fsigmoid{\masl{W}{xf}{i} \left( \dots \right) +\masl{W}{hf}{i}\vetl{h}{0}{i} + \vesl{b}{f}{i}} \odot 	\vetl{c}{0}{i}} \\
	& \quad + \fsigmoid{\masl{W}{xe}{i} \left( \dots\right)  +\masl{W}{he}{i}\vetl{h}{0}{i} + \vesl{b}{e}{i}} \\
	& \quad \ftanhclose{ \odot \ftanh{\masl{W}{xc}{i} \left(\dots \right)  +\masl{W}{hc}{i}\vetl{h}{0}{i} + \vesl{b}{c}{i}}} \\
    & = \dots  \\
	& = \fsigopen{\masl{W}{xo}{i} \dots \left\{  \fsigmoid{\masl{W}{xo}{j} \vetl{h}{1}{j-1} + \masl{W}{ho}{j}\vetl{h}{0}{j} + \vesl{b}{o}{j}} \right. }  \\
	& \qquad \odot \ftanhopen{ \fsigmoid{\masl{W}{xf}{j} \vetl{h}{1}{j-1}+\masl{W}{hf}{j}\vetl{h}{0}{j} + \vesl{b}{f}{j}} \odot \vetl{c}{0}{j}} \\
	& \qquad+ \fsigmoid{\masl{W}{xe}{j}\vetl{h}{1}{j-1}  +\masl{W}{he}{j}\vetl{h}{0}{j} + \vesl{b}{e}{j}} \\
	& \qquad \left.\ftanhclose{ \odot \ftanh{\masl{W}{xc}{j}\vetl{h}{1}{j-1}  +\masl{W}{hc}{j}\vetl{h}{0}{j} + \vesl{b}{c}{j}} } \right\}  \\
	& \quad \dots +\masl{W}{ho}{i}\vetl{h}{0}{i} + \vesl{b}{o}{i}  \\
	& \quad \odot \ftanhopen{ \fsigmoid{\masl{W}{xf}{i} \left( \dots \right) +\masl{W}{hf}{i}\vetl{h}{0}{i} + \vesl{b}{f}{i}} \odot 	\vetl{c}{0}{i}} \\
	& \quad + \fsigmoid{\masl{W}{xe}{i} \left( \dots\right)  +\masl{W}{he}{i}\vetl{h}{0}{i} + \vesl{b}{e}{i}} \\
	& \quad \ftanhclose{ \odot \ftanh{\masl{W}{xc}{i} \left(\dots \right)  +\masl{W}{hc}{i}\vetl{h}{0}{i} + \vesl{b}{c}{i}}} \\
    & = \dots 
\end{align*}
\begin{align*}
	& = \fsigmoid{\masl{W}{xo}{i}\vetl{h}{1}{i-1} + \masl{W}{ho}{i}\vetl{h}{0}{i} + \vesl{b}{o}{i}} \\
	& \quad \odot \ftanhopen{ \fsigmoid{\masl{W}{xf}{i}\vetl{h}{1}{i-1} + \masl{W}{hf}{i}\vetl{h}{0}{i} + \vesl{b}{f}{i}} \odot \vetl{c}{0}{i} } \\
	& \quad + \fsigmoid{\masl{W}{xe}{i} \vetl{h}{1}{i-1} +\masl{W}{he}{i}\vetl{h}{0}{i} + \vesl{b}{e}{i}} \\
	& \quad \ftanhclose{ \odot \ftanh{\masl{W}{xc}{i} \vetl{h}{1}{i-1}  +\masl{W}{hc}{i}\vetl{h}{0}{i} + \vesl{b}{c}{i}}} \\
	& = \vetl{h}{1}{i},
\end{align*}
where we used the initialization assumptions $\vehtb{h}{i-1}{i}=\vetl{h}{0}{i}$ and $\vehtb{c}{i-1}{i}=\vetl{c}{0}{i}$ for all $i=1\dots k$, and the definition of the hidden and cell state in Equations \eqref{eq:lstm-cellstate-short} and \eqref{eq:lstm-hiddenstate-short} for $j-1$ blocks, in the previous steps. In particular, we have for layer $k$ that $\vehtb{h}{i}{k}= \vetl{h}{1}{k}$, and using the same transformations as in \eqref{eq:dRNN_proof_yk} with RNNs, we obtain $\veht{y}{k} = \vet{y}{1}$. Furthermore, we obtained that:
\begin{align*}
\vehtb{c}{i}{i} &=   \fsigmoid{\mahsb{W}{xf}{i}\vet{x}{i} + \mahsb{W}{hf}{i}\veht{h}{i-1} + \vehsb{b}{f}{i}} \odot \vehtb{c}{i-1}{i}  \nonumber \\
&\quad+ \fsigmoid{\mahsb{W}{xe}{i}\vet{x}{i} + \mahsb{W}{he}{i}\veht{h}{i-1} + \vehsb{b}{e}{i} } \nonumber \\
&\quad \odot  \ftanh{\mahsb{W}{xc}{i}\vet{x}{i} + \mahsb{W}{hc}{i}\veht{h}{i-1} + \vehsb{b}{c}{i}} \\
&=   \fsigmoid{\masl{W}{xf}{i}\vehtb{h}{i-1}{i-1} + \masl{W}{hf}{i}\vehtb{h}{i-1}{i} + \vesl{b}{f}{i}} \odot \vetl{c}{0}{i}  \nonumber \\
&\quad+ \fsigmoid{\masl{W}{xe}{i}\vehtb{h}{i-1}{i-1} + \masl{W}{he}{i}\vehtb{h}{i-1}{i} + \vesl{b}{e}{i} } \nonumber \\
&\quad \odot  \ftanh{\masl{W}{xc}{i}\vehtb{h}{i-1}{i-1} + \masl{W}{hc}{i}\vehtb{h}{i-1}{i} + \vesl{b}{c}{i}} \\
&=\fsigmoid{\masl{W}{xf}{i}\vetl{h}{1}{i-1} + \masl{W}{hf}{i}\vetl{h}{0}{i} + \vesl{b}{f}{i}} \odot \vetl{c}{0}{i}  \nonumber \\
&\quad+ \fsigmoid{\masl{W}{xe}{i}\vetl{h}{1}{i-1} + \masl{W}{he}{i}\vetl{h}{0}{i} + \vesl{b}{e}{i} } \nonumber \\
&\quad \odot  \ftanh{\masl{W}{xc}{i}\vetl{h}{1}{i-1} + \masl{W}{hc}{i}\vetl{h}{0}{i} + \vesl{b}{c}{i}} \\
&=  \vetl{c}{1}{i} 
\end{align*}
Which concludes the basis of the induction.

Next, we assume that $\vehtb{h}{t+i-1}{i}  = \vetl{h}{t}{i}$ and $\vehtb{c}{t+i-1}{i}  = \vetl{c}{t}{i}$  for all $1 \le i \le k $ and $t \le T-1$, and prove that it holds for the hidden and cell states for all layers when $t=T$: $\vehtb{h}{T+i-1}{i}  = \vetl{h}{T}{i}$, $\forall 1\le i \le k$. Without loss of generality, we have for the hidden state \vehtb{h}{T+i-1}{i} in constrained weights single-layer d-LSTM  that,
\begin{align*}
\vehtb{h}{T+i-1}{i} &= \fsigmoid{\mahsb{W}{xo}{i}\vet{x}{T+i-1} + \mahsb{W}{ho}{i}\vehtb{h}{T+i-2}{i} + \vehsb{b}{o}{i}}  \\
& \odot \ftanh{\vehtb{c}{T+i-1}{i}} \\
& = \fsigmoid{\masl{W}{xo}{i}\vehtb{h}{T+i-2}{i-1} +\masl{W}{ho}{i}\vehtb{h}{T+i-2}{i} + \vesl{b}{o}{i}}  \\
& \odot \ftanh{\vehtb{c}{T+i-1}{i}} \\
& = \fsigmoid{\masl{W}{xo}{i}\vehtb{h}{T+i-2}{i-1} +\masl{W}{ho}{i}\vehtb{h}{T+i-2}{i} + \vesl{b}{o}{i}}  \\ 
& \quad \odot \ftanhopen{
	\fsigmoid{\masl{W}{xf}{i}\vehtb{h}{T+i-2}{i-1} +\masl{W}{hf}{i}\vehtb{h}{T+i-2}{i} + \vesl{b}{f}{i}}} \\
& \quad \odot \vehtb{c}{T+i-2}{i}  \\
& \quad	+\fsigmoid{\masl{W}{xe}{i}\vehtb{h}{T+i-2}{i-1} +\masl{W}{he}{i}\vehtb{h}{T+i-2}{i} + \vesl{b}{e}{i}} \\ 
& \quad \ftanhclose{\odot \ftanh{\masl{W}{xc}{i}\vehtb{h}{T+i-2}{i-1} +\masl{W}{hc}{i}\vehtb{h}{T+i-2}{i} + \vesl{b}{c}{i}} } \\
& = \dots 
\end{align*}
\begin{align*}
& = \fsigopen{\masl{W}{xo}{i} \dots \left\{	\fsigopen{\masl{W}{xo}{j} \left[ \dots \right.  } \right. } \\
& \; \qquad \fsigopen{\masl{W}{xo}{2} \fsigmoid{\masl{W}{xo}{1}\vet{x}{T} + \masl{W}{ho}{1}\vehtb{h}{T-1}{1} + \vesl{b}{o}{1}} } \\
& \enspace \qquad \odot \ftanhopen{ \fsigmoid{\masl{W}{xf}{1}\vet{x}{T} +\masl{W}{hf}{1}\vehtb{h}{T-1}{1} + \vesl{b}{f}{1}}} \\
& \enspace \qquad  \odot \vehtb{c}{T-1}{1} + \fsigmoid{\masl{W}{xe}{1}\vet{x}{T} +\masl{W}{he}{1}\vehtb{h}{T-1}{1} + \vesl{b}{e}{1}} \\
& \enspace \qquad \ftanhclose{ \odot \ftanh{\masl{W}{xc}{1}\vet{x}{T} +\masl{W}{hc}{1}\vehtb{h}{T-1}{1} + \vesl{b}{c}{1}} } \\
& \; \qquad \fsigclose{+ \masl{W}{ho}{2}\vehtb{h}{T}{2} + \vesl{b}{o}{2}} \\
& \; \qquad \odot \ftanhopen{ \fsigmoid{\masl{W}{xf}{2}\left(\dots \right)  +\masl{W}{hf}{2}\vehtb{h}{T}{2} + \vesl{b}{f}{2}} } \\
& \; \qquad \odot \vehtb{c}{T}{2} + \fsigmoid{\masl{W}{xe}{2}\left(\dots\right) +\masl{W}{he}{2}\vehtb{h}{T}{2} + \vesl{b}{e}{2}} \\
& \; \qquad \ftanhclose{\odot \ftanh{ \masl{W}{xc}{2}\left(\dots\right) +\masl{W}{hc}{2}\vehtb{h}{T}{2} + \vesl{b}{c}{2}}}  \\
& \enspace \quad \fsigclose{ \left. \dots \right] +\masl{W}{ho}{j}\vehtb{h}{T+j-2}{j} + \vesl{b}{o}{j}}  \\
& \enspace \quad \odot \ftanhopen{ \fsigmoid{\masl{W}{xf}{j} \left[\dots \right] +\masl{W}{hf}{j}\vehtb{h}{T+j-2}{j} + \vesl{b}{f}{j}} } \\
& \enspace \quad \odot \vehtb{c}{T+j-2}{j} + \fsigmoid{\masl{W}{xe}{j} \left[\dots \right]  +\masl{W}{he}{j}\vehtb{h}{T+j-2}{j} + \vesl{b}{e}{j}} \\
& \enspace \quad  \left. \ftanhclose{ \odot \ftanh{\masl{W}{xc}{j} \left[\dots \right]  +\masl{W}{hc}{j}\vehtb{h}{T+j-2}{j} + \vesl{b}{c}{j}} }  \right\}  \\
& \quad \dots  +\masl{W}{ho}{i}\vehtb{h}{T+i-2}{i} + \vesl{b}{o}{i}   \\
& \quad \odot \ftanhopen{ \fsigmoid{\masl{W}{xf}{i} \left( \dots \right) +\masl{W}{hf}{i}\vehtb{h}{T+i-2}{i} + \vesl{b}{f}{i}} } \\
& \quad	\odot \vehtb{c}{T+i-2}{i} + \fsigmoid{\masl{W}{xe}{i} \left( \dots\right)  +\masl{W}{he}{i}\vehtb{h}{T+i-2}{i} + \vesl{b}{e}{i}} \\
& \quad \ftanhclose{ \odot \ftanh{\masl{W}{xc}{i} \left(\dots \right)  +\masl{W}{hc}{i}\vehtb{h}{T+i-2}{i} + \vesl{b}{c}{i}}}
\end{align*}
From the inductive assumption we have that $\vehtb{h}{T+j-2}{j}  = \vetl{h}{T-1}{j}$ and $\vehtb{c}{T+j-2}{j}  = \vetl{c}{T-1}{j}$   for all $1 \le j \le k $, then it follows that
\begin{align*}
& = \fsigopen{\masl{W}{xo}{i} \dots \left\{	\fsigopen{\masl{W}{xo}{j} \left[ \dots \right.  } \right. } \\
& \; \qquad \fsigopen{\masl{W}{xo}{2} \fsigmoid{\masl{W}{xo}{1}\vet{x}{T} + \masl{W}{ho}{1}\vetl{h}{T-1}{1} + \vesl{b}{o}{1}} } \\
& \enspace \qquad \odot \ftanhopen{ \fsigmoid{\masl{W}{xf}{1}\vet{x}{T} +\masl{W}{hf}{1}\vetl{h}{T-1}{1} + \vesl{b}{f}{1}} \odot 	\vetl{c}{T-1}{1} } \\
& \enspace \qquad + \fsigmoid{\masl{W}{xe}{1}\vet{x}{T} +\masl{W}{he}{1}\vetl{h}{T-1}{1} + \vesl{b}{e}{1}} \\
& \enspace \qquad \ftanhclose{ \odot \ftanh{\masl{W}{xc}{1}\vet{x}{T} +\masl{W}{hc}{1}\vetl{h}{T-1}{1} + \vesl{b}{c}{1}} } \\
& \; \qquad \fsigclose{+ \masl{W}{ho}{2}\vetl{h}{T-1}{2} + \vesl{b}{o}{2}} \\
& \; \qquad \odot \ftanhopen{ \fsigmoid{\masl{W}{xf}{2}\left(\dots \right)  +\masl{W}{hf}{2}\vetl{h}{T-1}{2} + \vesl{b}{f}{2}}  \odot \vetl{c}{T-1}{2}} \\
& \; \qquad + \fsigmoid{\masl{W}{xe}{2}\left(\dots\right) +\masl{W}{he}{2}\vetl{h}{T-1}{2} + \vesl{b}{e}{2}} \\
& \; \qquad \ftanhclose{\odot \ftanh{ \masl{W}{xc}{2}\left(\dots\right) +\masl{W}{hc}{2}\vetl{h}{T-1}{2} + \vesl{b}{c}{2}}}  \\
& \enspace \quad \fsigclose{ \left. \dots \right] +\masl{W}{ho}{j}\vetl{h}{T-1}{j} + \vesl{b}{o}{j}}  \\
& \enspace \quad \odot \ftanhopen{ \fsigmoid{\masl{W}{xf}{j} \left[\dots \right] +\masl{W}{hf}{j}\vetl{h}{T-1}{j} + \vesl{b}{f}{j}} \odot \vetl{c}{T-1}{j} } \\ 
& \enspace \quad + \fsigmoid{\masl{W}{xe}{j} \left[\dots \right]  +\masl{W}{he}{j}\vetl{h}{T-1}{j} + \vesl{b}{e}{j}} \\
& \enspace \quad  \left. \ftanhclose{ \odot \ftanh{\masl{W}{xc}{j} \left[\dots \right]  +\masl{W}{hc}{j}\vetl{h}{T-1}{j} + \vesl{b}{c}{j}} }  \right\}  \\
& \quad \dots  +\masl{W}{ho}{i}\vetl{h}{T-1}{i} + \vesl{b}{o}{i}   \\
& \quad \odot \ftanhopen{ \fsigmoid{\masl{W}{xf}{i} \left( \dots \right) +\masl{W}{hf}{i}\vetl{h}{T-1}{i} + \vesl{b}{f}{i}} \odot 	\vetl{c}{T-1}{i} } \\
& \quad	+ \fsigmoid{\masl{W}{xe}{i} \left( \dots\right)  +\masl{W}{he}{i}\vetl{h}{T-1}{i} + \vesl{b}{e}{i}} \\
& \quad \ftanhclose{ \odot \ftanh{\masl{W}{xc}{i} \left(\dots \right)  +\masl{W}{hc}{i}\vetl{h}{T-1}{i} + \vesl{b}{c}{i}}}
\\
& = \fsigopen{\masl{W}{xo}{i} \dots \left\{ \fsigopen{\masl{W}{xo}{j} \left[\dots  \right.  } \right. } \\
& \; \qquad \fsigmoid{\masl{W}{xo}{2}\vetl{h}{T}{1} +  \masl{W}{ho}{2}\vetl{h}{T-1}{2} + \vesl{b}{o}{2}} \\
& \; \qquad \odot \ftanhopen{ \fsigmoid{\masl{W}{xf}{2}\vetl{h}{T}{1} +\masl{W}{hf}{2}\vetl{h}{T-1}{2} + \vesl{b}{f}{2}} \odot \vetl{c}{T-1}{2}} \\
& \; \qquad + \fsigmoid{\masl{W}{xe}{2}\vetl{h}{T}{1} +\masl{W}{he}{2}\vetl{h}{T-1}{2} + \vesl{b}{e}{2}} \\
& \; \qquad \ftanhclose{\odot \ftanh{ \masl{W}{xc}{2}\vetl{h}{T}{1} +\masl{W}{hc}{2}\vetl{h}{T-1}{2} + \vesl{b}{c}{2}}}  \\
& \enspace \quad \fsigclose{ \left. \dots \right] +\masl{W}{ho}{j}\vetl{h}{T-1}{j} + \vesl{b}{o}{j}}  \\
& \enspace \quad \odot \ftanhopen{ \fsigmoid{\masl{W}{xf}{j} \left[\dots \right] +\masl{W}{hf}{j}\vetl{h}{T-1}{j} + \vesl{b}{f}{j}}  \odot \vetl{c}{T-1}{j}} \\ 
& \enspace \quad + \fsigmoid{\masl{W}{xe}{j} \left[\dots \right]  +\masl{W}{he}{j}\vetl{h}{T-1}{j} + \vesl{b}{e}{j}} \\
& \enspace \quad \left. \ftanhclose{ \odot \ftanh{\masl{W}{xc}{j} \left[\dots \right]  +\masl{W}{hc}{j} \vetl{h}{T-1}{j} + \vesl{b}{c}{j}}	} \right\}  \\
& \quad \dots +\masl{W}{ho}{i}\vetl{h}{T-1}{i} + \vesl{b}{o}{i}  \\
& \quad \odot \ftanhopen{ \fsigmoid{\masl{W}{xf}{i} \left( \dots \right) +\masl{W}{hf}{i}\vetl{h}{T-1}{i} + \vesl{b}{f}{i}} \odot \vetl{c}{T-1}{i}} \\
& \quad + \fsigmoid{\masl{W}{xe}{i} \left( \dots\right)  +\masl{W}{he}{i}\vetl{h}{T-1}{i} + \vesl{b}{e}{i}} \\
& \quad \ftanhclose{ \odot \ftanh{\masl{W}{xc}{i} \left(\dots \right)  +\masl{W}{hc}{i}\vetl{h}{T-1}{i} + \vesl{b}{c}{i}}} 
\end{align*}
\begin{align*}
& = \dots  \\
& = \fsigopen{\masl{W}{xo}{i} \dots \left\{  \fsigmoid{\masl{W}{xo}{j} \vetl{h}{T}{j-1} + \masl{W}{ho}{j}\vetl{h}{T-1}{j} + \vesl{b}{o}{j}} \right. }  \\
& \qquad \odot \ftanhopen{ \fsigmoid{\masl{W}{xf}{j} \vetl{h}{T}{j-1}+\masl{W}{hf}{j}\vetl{h}{T-1}{j} + \vesl{b}{f}{j}} } \\
& \qquad \odot \vetl{c}{T-1}{j} + \fsigmoid{\masl{W}{xe}{j}\vetl{h}{T}{j-1}  +\masl{W}{he}{j}\vetl{h}{T-1}{j} + \vesl{b}{e}{j}} \\
& \qquad \left.\ftanhclose{ \odot \ftanh{\masl{W}{xc}{j}\vetl{h}{T}{j-1}  +\masl{W}{hc}{j}\vetl{h}{T-1}{j} + \vesl{b}{c}{j}} } \right\}  \\
& \quad \dots +\masl{W}{ho}{i}\vetl{h}{T-1}{i} + \vesl{b}{o}{i}  \\
& \quad \odot \ftanhopen{ \fsigmoid{\masl{W}{xf}{i} \left( \dots \right) +\masl{W}{hf}{i}\vetl{h}{T-1}{i} + \vesl{b}{f}{i}} } \\
& \quad \odot 	\vetl{c}{T-1}{i}+ \fsigmoid{\masl{W}{xe}{i} \left( \dots\right)  +\masl{W}{he}{i}\vetl{h}{T-1}{i} + \vesl{b}{e}{i}} \\
& \quad \ftanhclose{ \odot \ftanh{\masl{W}{xc}{i} \left(\dots \right)  +\masl{W}{hc}{i}\vetl{h}{T-1}{i} + \vesl{b}{c}{i}}} \\
& = \dots  \\
& = \fsigmoid{\masl{W}{xo}{i}\vetl{h}{T}{i-1} + \masl{W}{ho}{i}\vetl{h}{T-1}{i} + \vesl{b}{o}{i}} \\
& \quad \odot \ftanhopen{ \fsigmoid{\masl{W}{xf}{i}\vetl{h}{T}{i-1} + \masl{W}{hf}{i}\vetl{h}{T-1}{i} + \vesl{b}{f}{i}} } \\
& \quad \odot \vetl{c}{T-1}{i} + \fsigmoid{\masl{W}{xe}{i} \vetl{h}{T}{i-1} +\masl{W}{he}{i}\vetl{h}{T-1}{i} + \vesl{b}{e}{i}} \\
& \quad \ftanhclose{ \odot \ftanh{\masl{W}{xc}{i} \vetl{h}{T}{i-1}  +\masl{W}{hc}{i}\vetl{h}{T-1}{i} + \vesl{b}{c}{i}}} \\
& = \vetl{h}{T}{i},
\end{align*}
where we use the recurrent definition of the hidden and cell states in Equations \eqref{eq:lstm-cellstate-short} and \eqref{eq:lstm-hiddenstate-short}.
In particular, we obtained for $i=k$ that $\vehtb{h}{T+k-1}{k}=\vetl{h}{T}{k}$.
Applying the same steps as in the d-RNN proof in Eq. \eqref{eq:proof_steps_y}, we obtain $\veht{y}{T+k-1} = \vet{y}{T}$.
Last, we obtain for the cell state that
\begin{align*}
\vehtb{c}{T+i-1}{i} &=   \fsigmoid{\mahsb{W}{xf}{i}\vet{x}{T+i-1} + \mahsb{W}{hf}{i}\veht{h}{T+i-2} + \vehsb{b}{f}{i}} \odot \vehtb{c}{T+i-2}{i}  \nonumber \\
&\quad+ \fsigmoid{\mahsb{W}{xe}{i}\vet{x}{T+i-1} + \mahsb{W}{he}{i}\veht{h}{T+i-2} + \vehsb{b}{e}{i} } \nonumber \\
&\quad \odot  \ftanh{\mahsb{W}{xc}{i}\vet{x}{T+i-1} + \mahsb{W}{hc}{i}\veht{h}{T+i-2} + \vehsb{b}{c}{i}} \\
&= \fsigmoid{\masl{W}{xf}{i}\vehtb{h}{T+i-2}{i-1} + \masl{W}{hf}{i}\vehtb{h}{T+i-2}{i} + \vesl{b}{f}{i}} \odot \vetl{c}{T-1}{i}  \nonumber \\
&\quad+ \fsigmoid{\masl{W}{xe}{i}\vehtb{h}{T+i-2}{i-1} + \masl{W}{he}{i}\vehtb{h}{T+i-2}{i} + \vesl{b}{e}{i} } \nonumber \\
&\quad \odot  \ftanh{\masl{W}{xc}{i}\vehtb{h}{T+i-2}{i-1} + \masl{W}{hc}{i}\vehtb{h}{T+i-2}{i} + \vesl{b}{c}{i}} \\
&=   \fsigmoid{\masl{W}{xf}{i}\vetl{h}{T}{i-1} + \masl{W}{hf}{i}\vetl{h}{T-1}{i} + \vesl{b}{f}{i}} \odot \vetl{c}{T-1}{i}  \nonumber \\
&\quad+ \fsigmoid{\masl{W}{xe}{i}\vetl{h}{T}{i-1} + \masl{W}{he}{i}\vetl{h}{T-1}{i} + \vesl{b}{e}{i} } \nonumber \\
&\quad \odot  \ftanh{\masl{W}{xc}{i}\vetl{h}{T}{i-1} + \masl{W}{hc}{i}\vetl{h}{T-1}{i} + \vesl{b}{c}{i}} \\
&=  \vetl{c}{T}{i} 
\end{align*}

Which completes the proof.
\end{proof}

\section{Weight Constraints and Connections in d-RNN}
\begin{figure*}[tb]%
\vskip 0.2in
\begin{center}
	\centerline{\includegraphics[width=\textwidth]{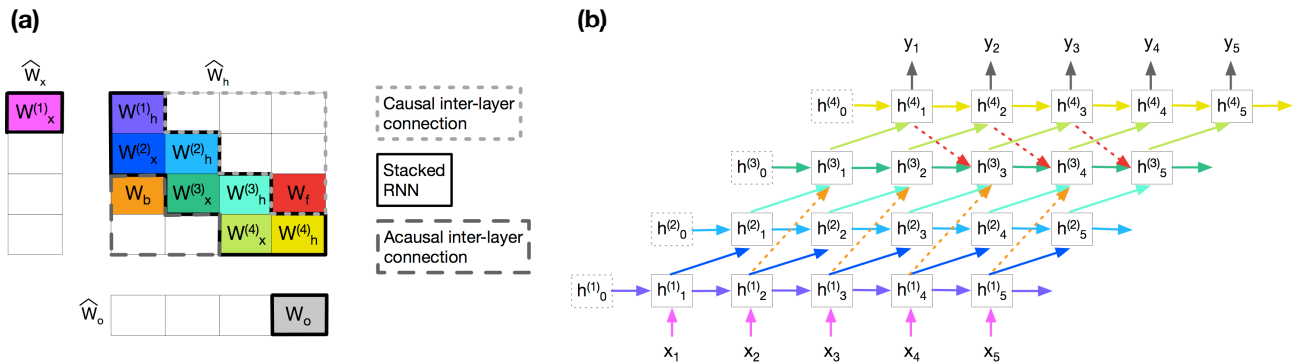}}
\caption{(a) Weights of the single-layer and weight constrained d-RNN that are equivalent to connections in the stacked RNN from Figure 2.
\ (b) Connections in the d-RNN based on the weight matrix in (a). The d-RNN is depicted as it would be the stacked RNN. The hidden states are delayed in time with respect to the stacked network.}
\label{fig:WeightMatrix2}
\end{center}
\vskip -0.2in
\end{figure*}
Figure \ref{fig:WeightMatrix2} shows the weight constraints imposed to achieve equivalence between the stacked RNN and single-layer d-RNN, and a visualization of the d-RNN as connections in the stacked RNN. Figure \ref{fig:WeightMatrix2}(b) depicts the delay (or ``shift'') of all the hidden states as they would be computed in the stacked RNN. Each layer is equivalent to a shift by one timestep. 

\section{Additional Plots for Error Maps}
Figure \ref{fig:sinus-std} present the standard deviation diagrams for the error maps in Figure 5. 

\begin{figure*}[tb]
\vskip 0.2in
\begin{center}
\subfigure[LSTM]{\label{fig:sinus-LSTM-std} \includegraphics[width=0.35\linewidth]{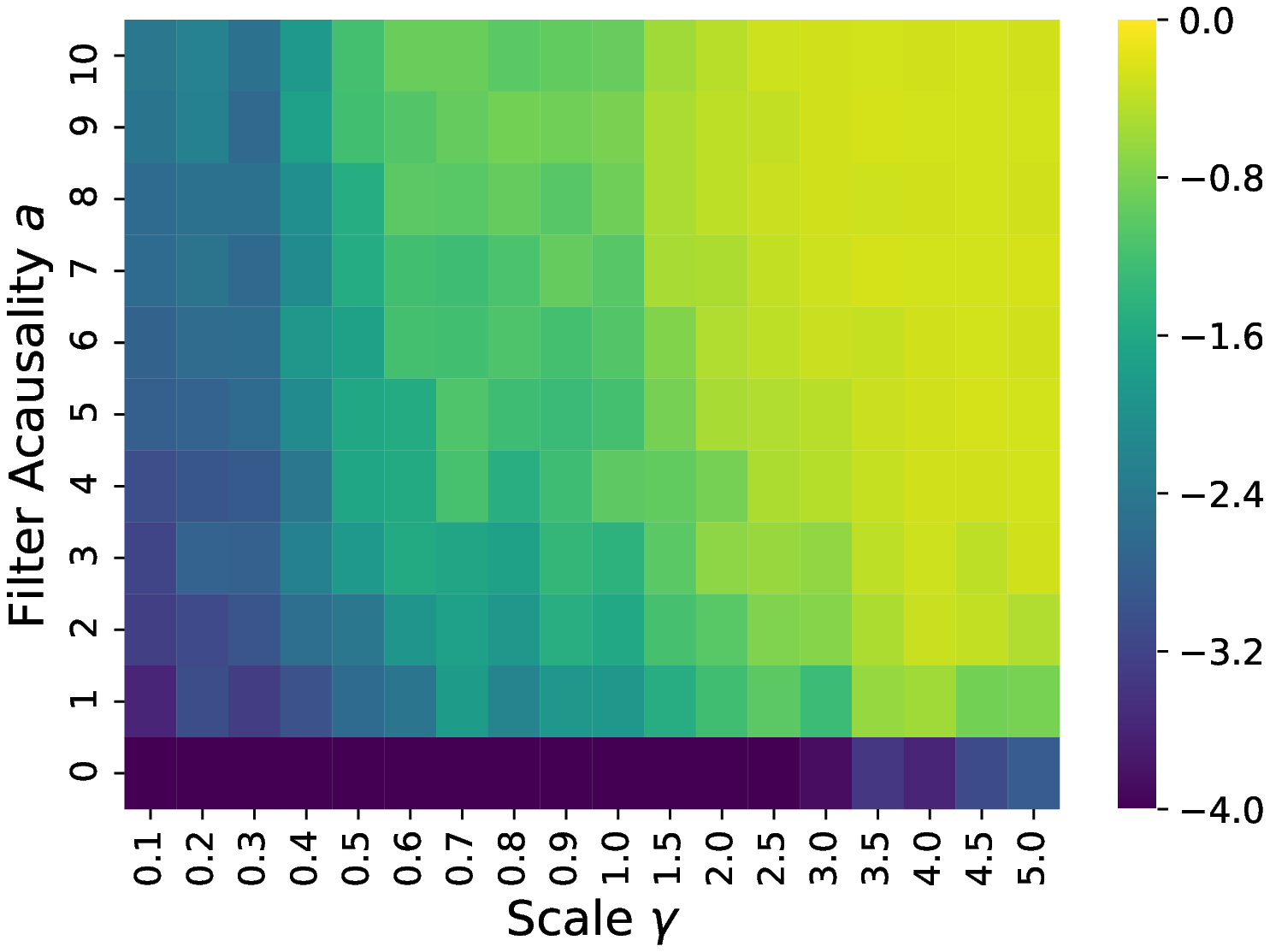} \includegraphics[width=0.35\linewidth]{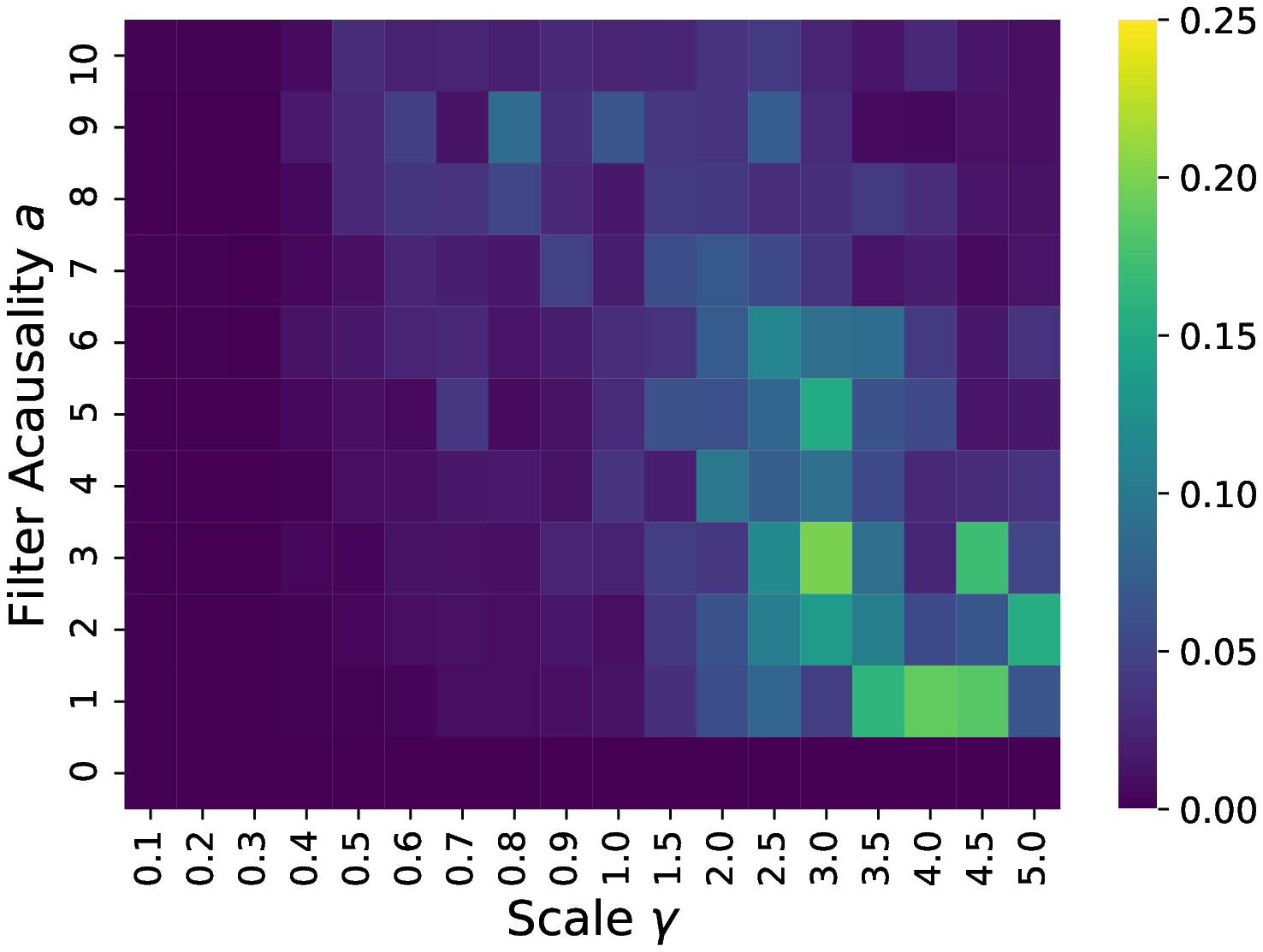}} \\
\subfigure[Bi-LSTM]{\label{fig:sinus-BiLSTM-std}\includegraphics[width=0.35\linewidth]{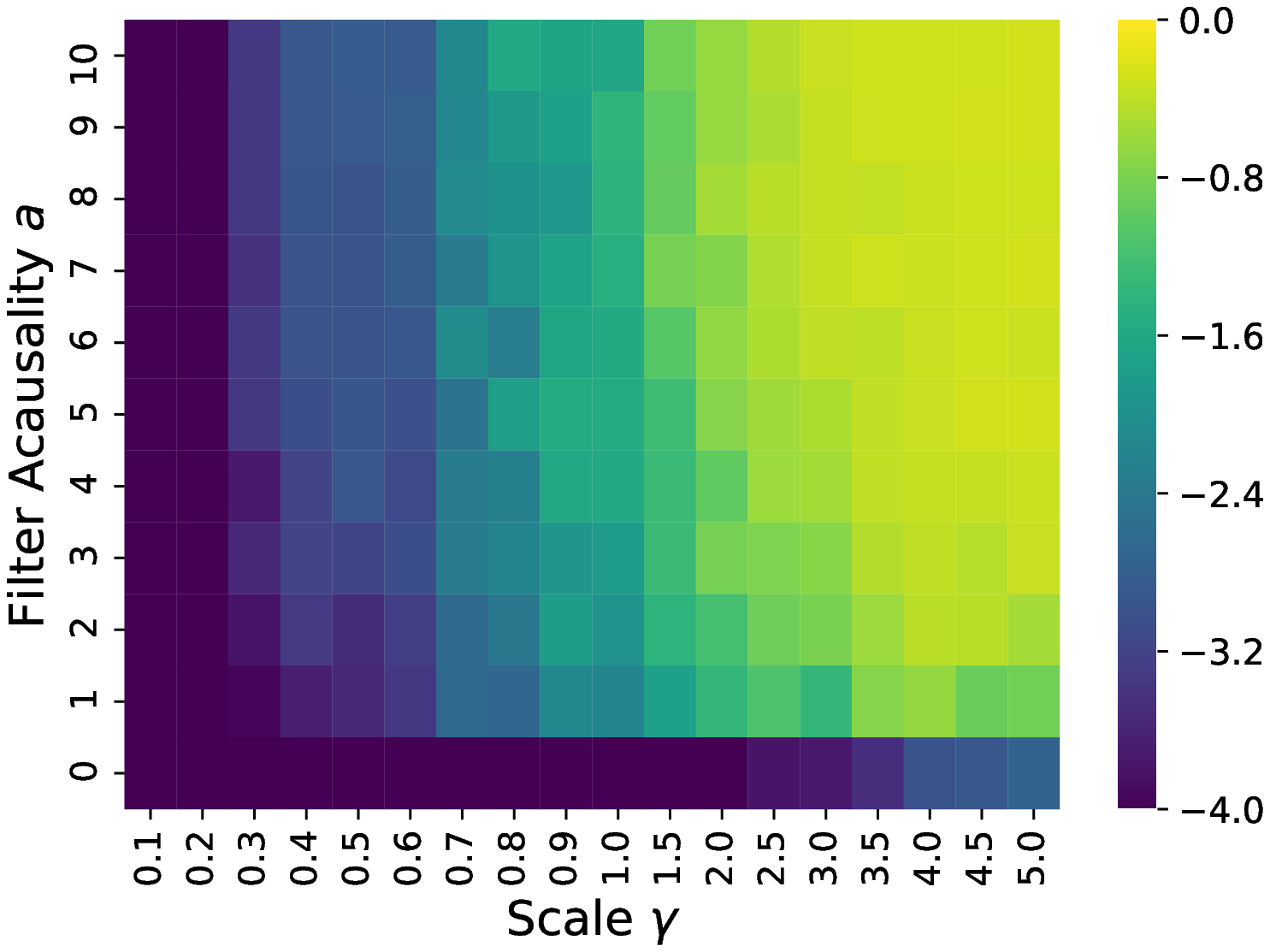}\includegraphics[width=0.35\linewidth]{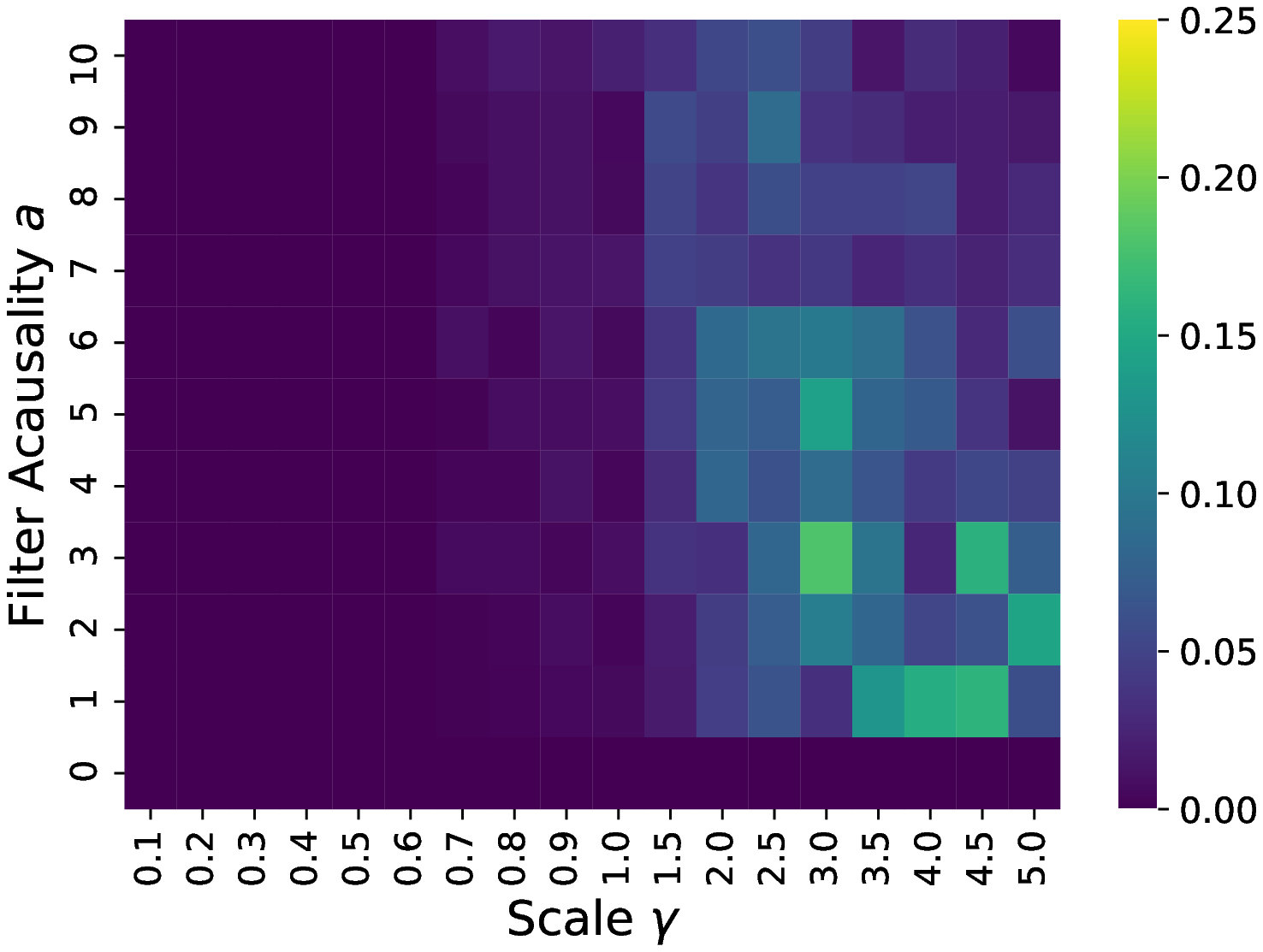}}\\ 
\subfigure[d-LSTM with delay=5]{\label{fig:sinus-dLSTM-5-std}\includegraphics[width=0.35\linewidth]{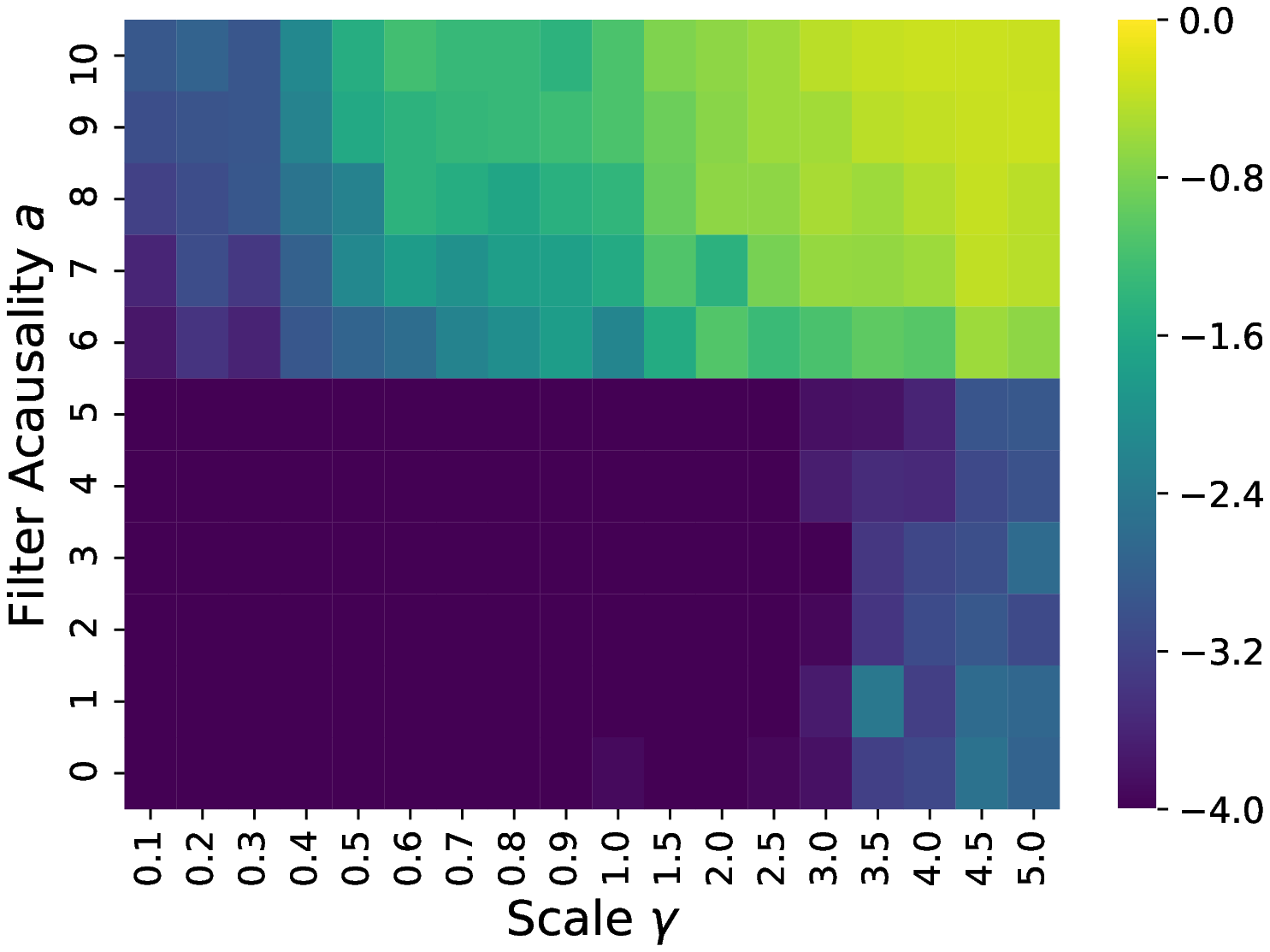}\includegraphics[width=0.35\linewidth]{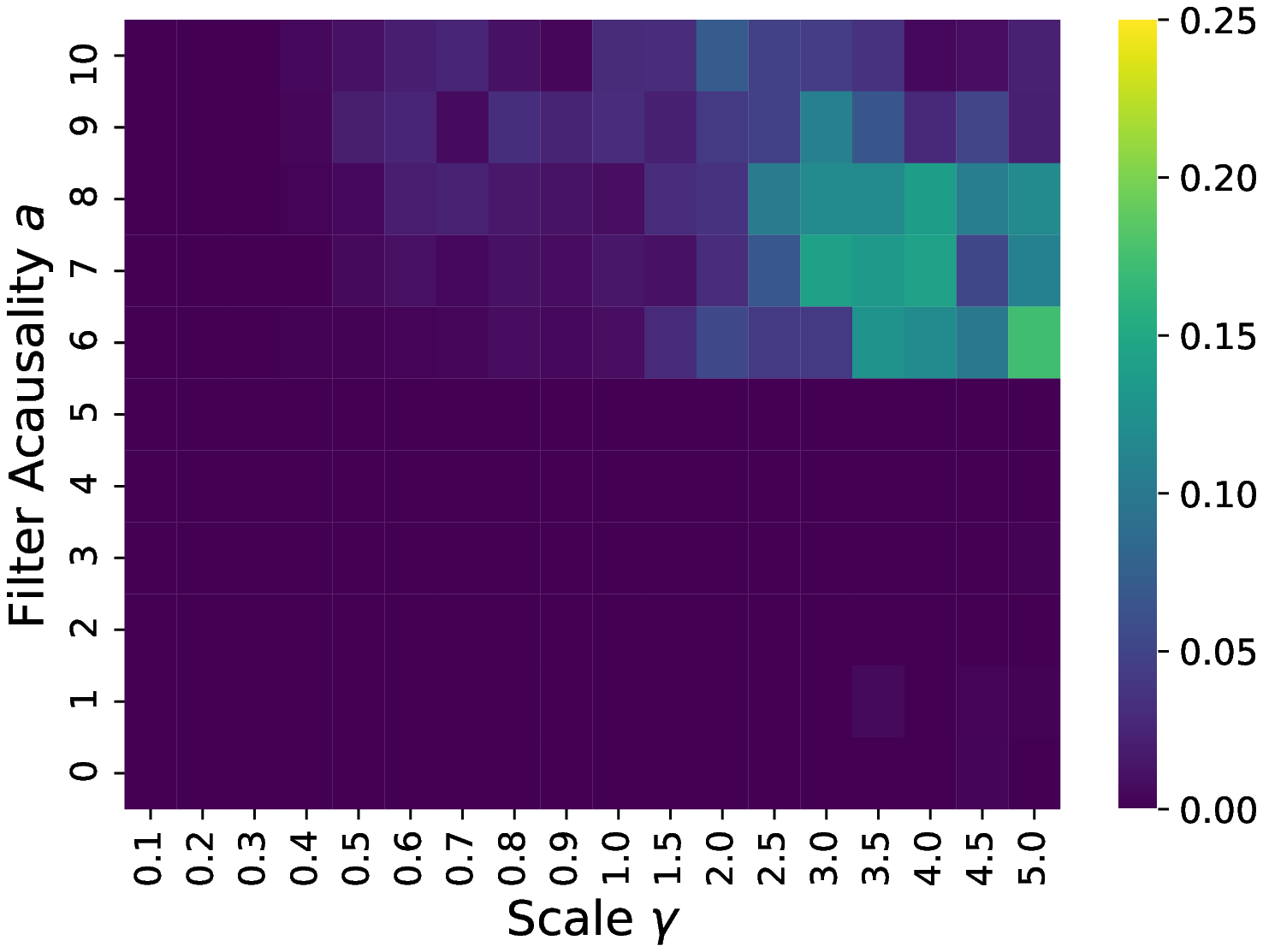}} \\
\subfigure[d-LSTM with delay=10]{\label{fig:sinus-dLSTM-10-std}\includegraphics[width=0.35\linewidth]{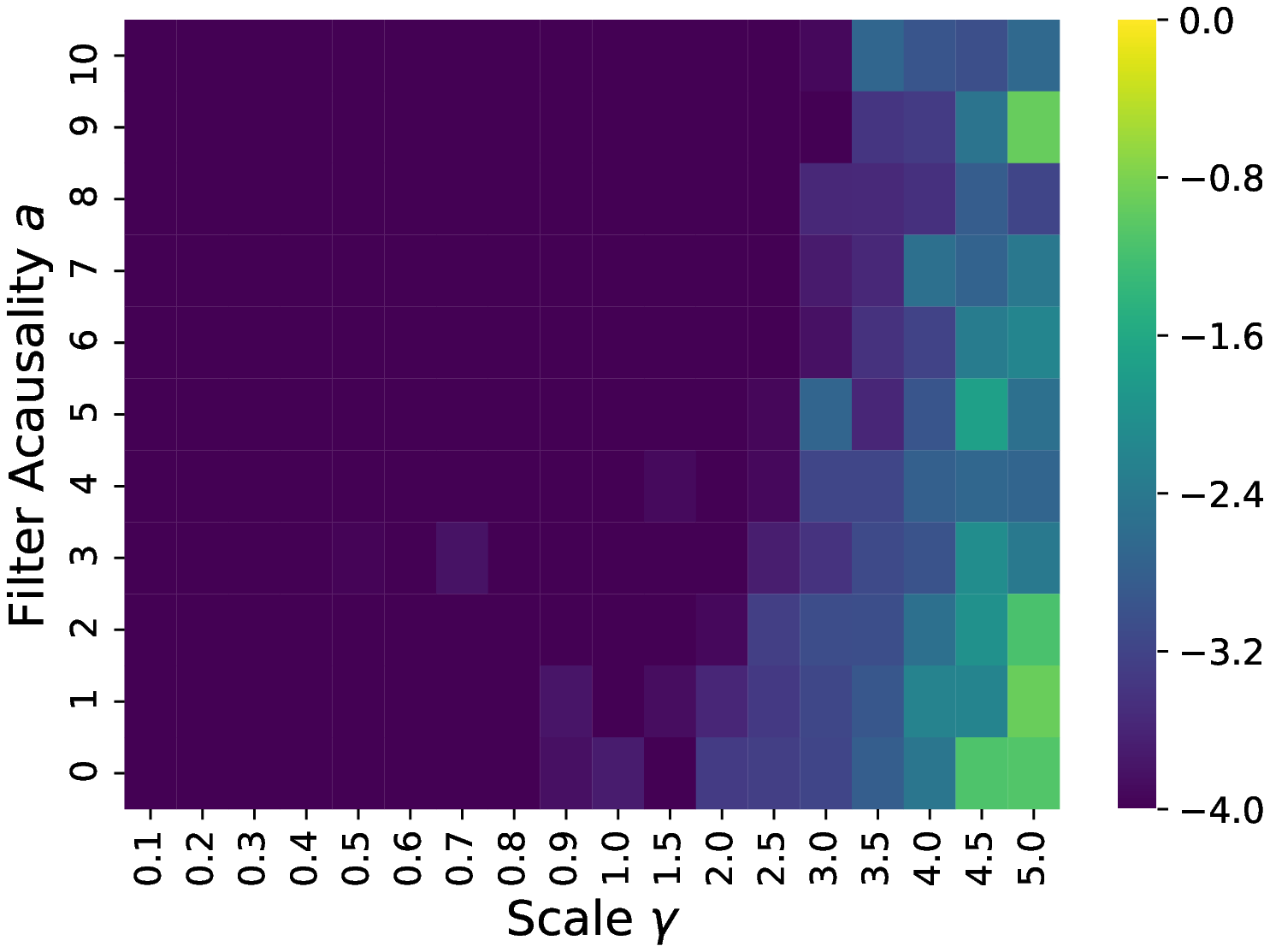}\includegraphics[width=0.35\linewidth]{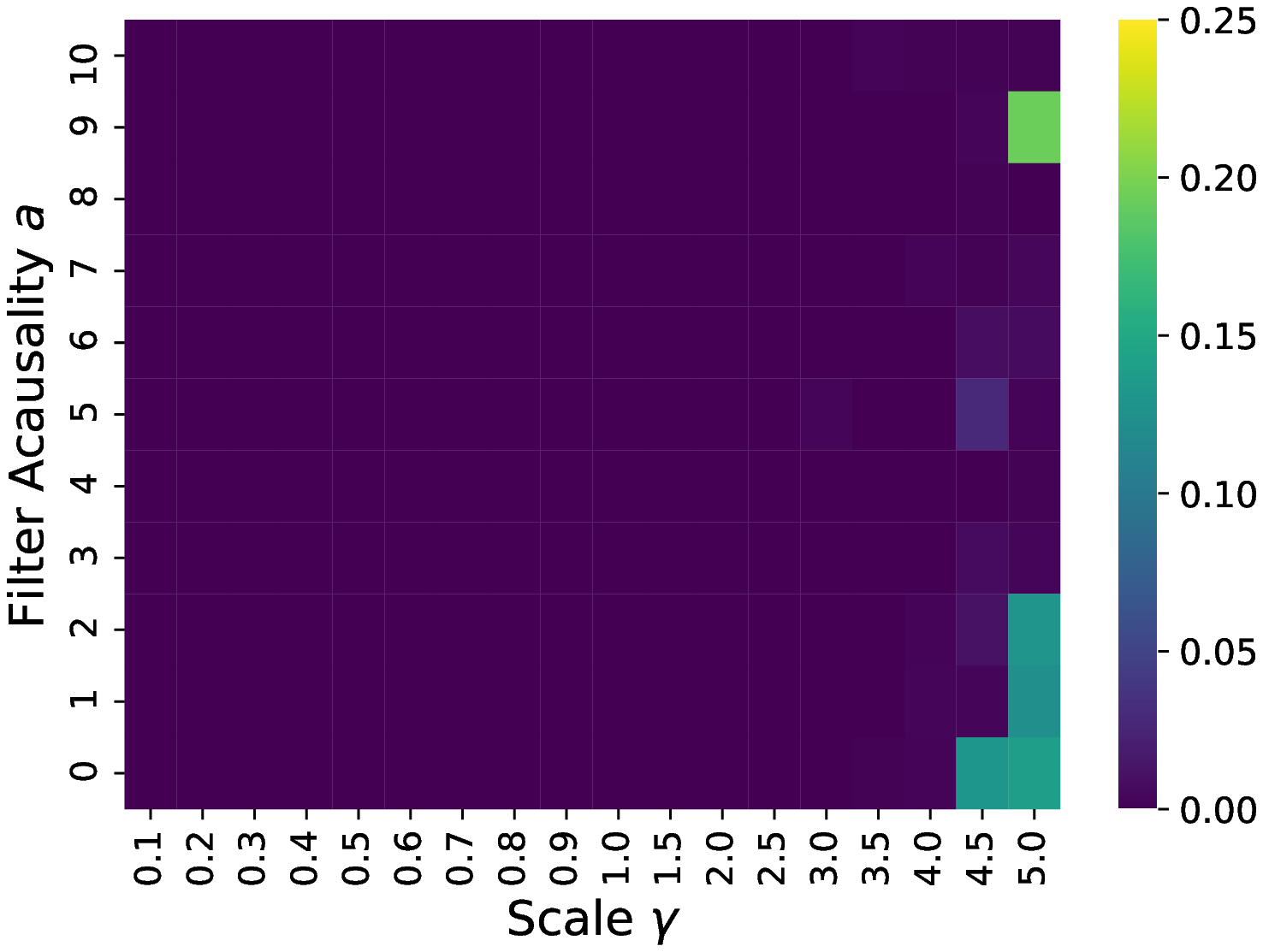}} \\ 
	\caption{Error maps presented in Figure 4 (left column) together with their standard deviation figures.}
\label{fig:sinus-std}
\end{center}
\vskip -0.2in
\end{figure*}

\section{Masked Character-Level Language Modeling: Additional Results}
In Table \ref{tbl:mlm}, we include additional results for smaller networks of the masked language model task. We sampled more delay values for d-LSTMs, but the general conclusions remain the same: intermediate values of delay achieve the lowest BPC. Forward-pass runtimes across delay values show a small increase with larger delays, but the increment is relatively flat compared to stacked LSTMs or (stacked) Bi-LSTMs as they increase in depth. For these experiments, we also used a batch of 128 sequences, and an embedding of dimension 10. 

\begin{table*}[tb]
	\caption{Performance for smaller networks on the masked character-level language modeling task. Mean and standard deviation values are computed over 5 repetitions of training and inference runtime on the test set.}
	\label{tbl:mlm}
	\vskip 0.1in
	\begin{center}
		\begin{small}
			\begin{sc}
				\begin{tabular}{lccccccr}
					\toprule
					Model & Layers & Delay & Units & Params. & Val. BPC & Test BPC & Runtime  \\
					\midrule
				        LSTM & 1 &  - &  512 & 1087283 & $2.139 \pm 0.005$ & $\mathbf{2.195 \pm 0.002}$ & $ 2.85ms \pm 0.14$ \\
				        LSTM & 2 &  - &  298 & 1090689 & $2.156 \pm 0.003$ & $2.215 \pm 0.002$ & $ 6.69ms \pm 0.27$ \\
				        LSTM & 5 &  - &  172 & 1083735 & $2.199 \pm 0.016$ & $2.255 \pm 0.015$ & $11.32ms \pm 0.05$ \\
				        \midrule
				     Bi-LSTM & 1 &  - &  360 & 1091107 & $1.130 \pm 0.003$ & $1.187 \pm 0.004$ & $ 5.82ms \pm 0.18$ \\
				     Bi-LSTM & 2 &  - &  182 & 1090487 & $0.800 \pm 0.004$ & $0.846 \pm 0.005$ & $11.08ms \pm 0.59$ \\
				     Bi-LSTM & 5 &  - &  102 & 1104151 & $0.796 \pm 0.007$ & $\mathbf{0.841 \pm 0.006}$ & $23.94ms \pm 0.17$ \\
				        \midrule
				      d-LSTM & 1 &  1 &  512 & 1087283 & $1.470 \pm 0.002$ & $1.518 \pm 0.003$ & $ 2.80ms \pm 0.02$ \\
				     d-LSTM & 1 &  2 &  512 & 1087283 & $1.162 \pm 0.004$ & $1.208 \pm 0.003$ & $ 2.81ms \pm 0.01$ \\
				     d-LSTM & 1 &  3 &  512 & 1087283 & $0.995 \pm 0.002$ & $1.039 \pm 0.002$ & $ 3.02ms \pm 0.23$ \\
				      d-LSTM & 1 &  5 &  512 & 1087283 & $0.877 \pm 0.001$ & $0.920 \pm 0.003$ & $ 3.01ms \pm 0.22$ \\
				      d-LSTM & 1 &  8 &  512 & 1087283 & $0.859 \pm 0.002$ & $\mathbf{0.905 \pm 0.003}$ & $ 3.04ms \pm 0.19$ \\
				      d-LSTM & 1 & 10 &  512 & 1087283 & $0.889 \pm 0.004$ & $0.935 \pm 0.005$ & $ 3.22ms \pm 0.18$ \\
				     d-LSTM & 1 & 15 &  512 & 1087283 & $0.971 \pm 0.004$ & $1.014 \pm 0.002$ & $ 3.17ms \pm 0.05$ \\
					\bottomrule
				\end{tabular}
			\end{sc}
		\end{small}
	\end{center}
	\vskip -0.1in
\end{table*}

\section{Part-of-Speech Tagging: Additional Details and Results}
In this section, we include more details about the dataset and the results of all the combinations for the Parts-Of-Speech experiment. 
We used treebanks from Universal Dependencies (UD)~\cite{UD} version 2.3. We selected the English EWT treebank\footnote{\url{https://github.com/UniversalDependencies/UD_English-EWT/tree/r2.3}}~\cite{UDEngEWT} (254,854 words), French GSD treebank\footnote{\url{https://github.com/UniversalDependencies/UD_French-GSD/tree/r2.3}} (411,465 words), and German GSD treebank\footnote{\url{https://github.com/UniversalDependencies/UD_German-GSD/tree/r2.3}} (297,836 words) based on the quality assigned by the UD authors. 
We follow the partitioning onto training, validation and test datasets as pre-defined in UD. All treebanks use the same POS tag set containing 17 tags. We use the Polyglot project~\cite{polyglot} word embeddings (64 dimensions). We build our own alphabets based on the most frequent 100 characters in the vocabularies. All the networks have a 100-dimensional character-level embedding, which is trained with the network. We use a batch size of 32 sentences.

Results for German, English, and French can be found in Tables \ref{tbl:pos-de-full}, \ref{tbl:pos-en-full}, and \ref{tbl:pos-fr-full}, respectively. The best result that does not use a bidirectional network is marked in bold for each language.

\begin{table*}[tb]
	\caption{Parts-of-Speech results for German. The table shows all possible combinations of delays or bidirectional LSTM networks. The best forward-only network is marked in bold.}
	\label{tbl:pos-de-full}
\vskip 0.15in
\begin{center}
	\begin{small}
		\begin{sc}
\begin{tabular}{cccc}
	    \toprule
	 Character-level network & Word-level network & Validation Accuracy & Test Accuracy \\
	    \midrule
Bi-LSTM&Bi-LSTM & $93.88 \pm 0.13$ &  $93.15 \pm 0.08$  \\ 
Bi-LSTM&LSTM &  $92.00 \pm 0.16$ &  $91.50 \pm 0.05$  \\
Bi-LSTM&d-LSTM with delay=1 &  $93.32 \pm 0.23$ &  $92.81 \pm 0.14$  \\
Bi-LSTM&d-LSTM with delay=2 &  $93.15 \pm 0.06$ &  $92.67 \pm 0.08$  \\
Bi-LSTM&d-LSTM with delay=3 &  $92.82 \pm 0.14$ &  $92.25 \pm 0.16$  \\
Bi-LSTM&d-LSTM with delay=4 &  $92.41 \pm 0.12$ &  $91.95 \pm 0.17$  \\
Bi-LSTM&d-LSTM with delay=5 &  $91.86 \pm 0.11$ &  $91.57 \pm 0.20$  \\
	    \midrule
LSTM&Bi-LSTM &  $93.96 \pm 0.12$ &  $93.43 \pm 0.07$  \\ 
LSTM&LSTM &  $92.05 \pm 0.16$ &  $91.58 \pm 0.11$  \\
LSTM&d-LSTM with delay=1 &  $93.46 \pm 0.16$ &  $92.71 \pm 0.11$  \\
LSTM&d-LSTM with delay=2 &  $93.13 \pm 0.10$ &  $92.61 \pm 0.26$  \\
LSTM&d-LSTM with delay=3 &  $92.91 \pm 0.13$ &  $92.38 \pm 0.15$  \\
LSTM&d-LSTM with delay=4 &  $92.56 \pm 0.17$ &  $92.06 \pm 0.19$  \\
	    \midrule
d-LSTM with delay=1&Bi-LSTM &  $93.93 \pm 0.06$ &  $93.39 \pm 0.18$  \\ 
d-LSTM with delay=1&LSTM &  $92.04 \pm 0.11$ &  $91.58 \pm 0.14$  \\
d-LSTM with delay=1&d-LSTM with delay=1 &  $93.48 \pm 0.31$ &  $\mathbf{92.87 \pm 0.24}$  \\
d-LSTM with delay=1&d-LSTM with delay=2 &  $93.11 \pm 0.18$ &  $92.54 \pm 0.08$  \\
d-LSTM with delay=1&d-LSTM with delay=3 &  $92.85 \pm 0.14$ &  $92.28 \pm 0.19$  \\
d-LSTM with delay=1&d-LSTM with delay=4 &  $92.50 \pm 0.12$ &  $92.11 \pm 0.19$  \\
\midrule
d-LSTM with delay=3&Bi-LSTM & $94.00 \pm 0.17$ &  $93.32 \pm 0.18$  \\ 
d-LSTM with delay=3&LSTM &  $92.10 \pm 0.24$ &  $91.61 \pm 0.18$  \\
d-LSTM with delay=3&d-LSTM with delay=1 &  $93.29 \pm 0.09$ &  $92.68 \pm 0.09$  \\
d-LSTM with delay=3&d-LSTM with delay=2 &  $93.09 \pm 0.21$ &  $92.59 \pm 0.16$  \\
d-LSTM with delay=3&d-LSTM with delay=3 &  $92.86 \pm 0.24$ &  $92.42 \pm 0.16$  \\
d-LSTM with delay=3&d-LSTM with delay=4 &  $92.53 \pm 0.17$ &  $92.08 \pm 0.18$  \\
\midrule
d-LSTM with delay=5&Bi-LSTM & $93.88 \pm 0.17$ &  $93.27 \pm 0.06$  \\ 
d-LSTM with delay=5&LSTM &  $91.88 \pm 0.18$ &  $91.54 \pm 0.11$  \\
d-LSTM with delay=5&d-LSTM with delay=1 &  $93.31 \pm 0.14$ &  $92.74 \pm 0.10$  \\
d-LSTM with delay=5&d-LSTM with delay=2 &  $93.17 \pm 0.13$ &  $92.57 \pm 0.17$  \\
d-LSTM with delay=5&d-LSTM with delay=3 &  $92.84 \pm 0.19$ &  $92.25 \pm 0.10$  \\
d-LSTM with delay=5&d-LSTM with delay=4 &  $92.50 \pm 0.22$ &  $91.96 \pm 0.19$  \\
    \bottomrule
\end{tabular}
\end{sc}
\end{small}
\end{center}
\vskip -0.1in
\end{table*}

\begin{table*}[tb]
	\caption{Parts-of-Speech results for English. The table shows all possible combinations of delays or bidirectional LSTM networks. The best forward-only network is marked in bold.}
\label{tbl:pos-en-full}
\vskip 0.15in
\begin{center}
	\begin{small}
		\begin{sc}
\begin{tabular}{cccc}
	\toprule
	Character-level network & Word-level network & Validation Accuracy & Test Accuracy \\
	\midrule
 Bi-LSTM&Bi-LSTM & $94.85 \pm 0.05$ &  $94.84 \pm 0.08$  \\ 
 Bi-LSTM&LSTM &  $91.90 \pm 0.12$ &  $92.05 \pm 0.09$  \\
 Bi-LSTM&d-LSTM with delay=1 &  $94.47 \pm 0.06$ &  $94.41 \pm 0.05$  \\
 Bi-LSTM&d-LSTM with delay=2 &  $94.17 \pm 0.13$ &  $94.14 \pm 0.10$  \\
 Bi-LSTM&d-LSTM with delay=3 &  $93.70 \pm 0.07$ &  $93.87 \pm 0.07$  \\
 Bi-LSTM&d-LSTM with delay=4 &  $93.11 \pm 0.14$ &  $93.26 \pm 0.08$  \\
 Bi-LSTM&d-LSTM with delay=5 &  $92.54 \pm 0.16$ &  $92.70 \pm 0.10$  \\
\midrule	
LSTM&Bi-LSTM & $95.03 \pm 0.14$ &  $94.99 \pm 0.15$  \\ 
LSTM&LSTM &  $92.05 \pm 0.13$ &  $92.14 \pm 0.10$  \\
LSTM&d-LSTM with delay=1 &  $94.53 \pm 0.08$ &  $94.58 \pm 0.11$  \\
LSTM&d-LSTM with delay=2 &  $94.29 \pm 0.05$ &  $94.28 \pm 0.05$  \\
LSTM&d-LSTM with delay=3 &  $93.81 \pm 0.11$ &  $93.85 \pm 0.12$  \\
LSTM&d-LSTM with delay=4 &  $93.39 \pm 0.12$ &  $93.55 \pm 0.10$  \\
\midrule	
d-LSTM with delay=1&Bi-LSTM & $94.94 \pm 0.07$ &  $94.95 \pm 0.06$  \\ 
d-LSTM with delay=1&LSTM &  $91.96 \pm 0.16$ &  $92.09 \pm 0.10$  \\
d-LSTM with delay=1&d-LSTM with delay=1 &  $94.57 \pm 0.08$ &  $\mathbf{94.57 \pm 0.14}$  \\
d-LSTM with delay=1&d-LSTM with delay=2 &  $94.29 \pm 0.12$ &  $94.37 \pm 0.08$  \\
d-LSTM with delay=1&d-LSTM with delay=3 &  $93.86 \pm 0.05$ &  $93.84 \pm 0.10$  \\
d-LSTM with delay=1&d-LSTM with delay=4 &  $93.35 \pm 0.10$ &  $93.56 \pm 0.13$  \\
\midrule
 d-LSTM with delay=3&Bi-LSTM &   $94.98 \pm 0.09$ &  $94.91 \pm 0.10$  \\ 
 d-LSTM with delay=3&LSTM &  $91.96 \pm 0.08$ &  $92.08 \pm 0.10$  \\
 d-LSTM with delay=3&d-LSTM with delay=1 &  $94.47 \pm 0.03$ &  $94.51 \pm 0.10$  \\
 d-LSTM with delay=3&d-LSTM with delay=2 &  $94.21 \pm 0.05$ &  $94.18 \pm 0.03$  \\
 d-LSTM with delay=3&d-LSTM with delay=3 &  $93.80 \pm 0.13$ &  $93.88 \pm 0.13$  \\
 d-LSTM with delay=3&d-LSTM with delay=4 &  $93.23 \pm 0.13$ &  $93.38 \pm 0.11$  \\
\midrule
 d-LSTM with delay=5&Bi-LSTM & $94.90 \pm 0.07$ &  $94.87 \pm 0.09$  \\ 
 d-LSTM with delay=5&LSTM &  $91.84 \pm 0.11$ &  $91.98 \pm 0.20$  \\
 d-LSTM with delay=5&d-LSTM with delay=1 &  $94.36 \pm 0.09$ &  $94.44 \pm 0.08$  \\
 d-LSTM with delay=5&d-LSTM with delay=2 &  $94.05 \pm 0.07$ &  $94.19 \pm 0.05$  \\
 d-LSTM with delay=5&d-LSTM with delay=3 &  $93.61 \pm 0.07$ &  $93.76 \pm 0.05$  \\
 d-LSTM with delay=5&d-LSTM with delay=4 &  $93.14 \pm 0.04$ &  $93.27 \pm 0.12$  \\
	    \bottomrule
\end{tabular}
\end{sc}
\end{small}
\end{center}
\vskip -0.1in
\end{table*}

\begin{table*}[tb]
	\caption{Parts-of-Speech results for French. The table shows all possible combinations of delays or bidirectional LSTM networks. The best forward-only network is marked in bold.}
	\label{tbl:pos-fr-full}
\vskip 0.15in
\begin{center}
	\begin{small}
		\begin{sc}
	\begin{tabular}{cccc}
		\toprule
		Character-level network & Word-level network & Validation Accuracy & Test Accuracy \\
		\midrule
 Bi-LSTM&Bi-LSTM &  $97.63 \pm 0.06$ &  $97.22 \pm 0.11$  \\ 
 Bi-LSTM&LSTM &  $96.67 \pm 0.05$ &  $96.15 \pm 0.17$  \\
 Bi-LSTM&d-LSTM with delay=1 &  $97.48 \pm 0.02$ &  $96.98 \pm 0.05$  \\
 Bi-LSTM&d-LSTM with delay=2 &  $97.41 \pm 0.02$ &  $96.91 \pm 0.12$  \\
 Bi-LSTM&d-LSTM with delay=3 &  $97.31 \pm 0.05$ &  $96.84 \pm 0.09$  \\
 Bi-LSTM&d-LSTM with delay=4 &  $97.12 \pm 0.05$ &  $96.61 \pm 0.06$  \\
Bi-LSTM&d-LSTM with delay=5 &  $96.88 \pm 0.10$ &  $96.20 \pm 0.14$  \\
		\midrule
 LSTM&Bi-LSTM &  $97.70 \pm 0.07$ &  $97.19 \pm 0.09$  \\ 
 LSTM&LSTM &  $96.67 \pm 0.07$ &  $96.10 \pm 0.11$  \\
 LSTM&d-LSTM with delay=1 &  $97.49 \pm 0.07$ &  $97.03 \pm 0.07$  \\
 LSTM&d-LSTM with delay=2 &  $97.49 \pm 0.05$ &  $97.00 \pm 0.06$  \\
LSTM&d-LSTM with delay=3 &  $97.34 \pm 0.04$ &  $96.89 \pm 0.09$  \\
 LSTM&d-LSTM with delay=4 &  $97.16 \pm 0.06$ &  $96.66 \pm 0.15$  \\
		\midrule
d-LSTM with delay=1&Bi-LSTM & $97.67 \pm 0.07$ &  $97.23 \pm 0.12$  \\ 
 d-LSTM with delay=1&LSTM &  $96.66 \pm 0.06$ &  $95.97 \pm 0.07$  \\
d-LSTM with delay=1&d-LSTM with delay=1 &  $97.49 \pm 0.04$ &  $\mathbf{97.04 \pm 0.13}$  \\
 d-LSTM with delay=1&d-LSTM with delay=2 &  $97.43 \pm 0.05$ &  $96.98 \pm 0.05$  \\
d-LSTM with delay=1&d-LSTM with delay=3 &  $97.36 \pm 0.08$ &  $96.80 \pm 0.10$  \\
 d-LSTM with delay=1&d-LSTM with delay=4 &  $97.22 \pm 0.06$ &  $96.57 \pm 0.10$  \\
		\midrule
 d-LSTM with delay=3&Bi-LSTM &  $97.67 \pm 0.08$ &  $97.21 \pm 0.08$  \\ 
 d-LSTM with delay=3&LSTM &  $96.67 \pm 0.07$ &  $95.98 \pm 0.14$  \\
 d-LSTM with delay=3&d-LSTM with delay=1 &  $97.52 \pm 0.04$ &  $97.02 \pm 0.09$  \\
 d-LSTM with delay=3&d-LSTM with delay=2 &  $97.44 \pm 0.02$ &  $96.97 \pm 0.12$  \\
 d-LSTM with delay=3&d-LSTM with delay=3 &  $97.28 \pm 0.04$ &  $96.74 \pm 0.07$  \\
 d-LSTM with delay=3&d-LSTM with delay=4 &  $97.13 \pm 0.05$ &  $96.57 \pm 0.09$  \\
		\midrule
 d-LSTM with delay=5&Bi-LSTM &  $97.61 \pm 0.03$ &  $97.12 \pm 0.06$  \\ 
d-LSTM with delay=5&LSTM &  $96.64 \pm 0.06$ &  $96.08 \pm 0.08$  \\
 d-LSTM with delay=5&d-LSTM with delay=1 &  $97.46 \pm 0.02$ &  $96.96 \pm 0.13$  \\
 d-LSTM with delay=5&d-LSTM with delay=2 &  $97.41 \pm 0.06$ &  $96.87 \pm 0.06$  \\
 d-LSTM with delay=5&d-LSTM with delay=3 &  $97.36 \pm 0.05$ &  $96.82 \pm 0.07$  \\
 d-LSTM with delay=5&d-LSTM with delay=4 &  $97.15 \pm 0.05$ &  $96.51 \pm 0.07$  \\
	    \bottomrule
\end{tabular}
\end{sc}
\end{small}
\end{center}
\vskip -0.1in
\end{table*}		
\fi

\end{document}